\documentclass[11pt]{article}

\usepackage[preprint]{acl}
\usepackage{times}
\usepackage{latexsym}

\usepackage{amsmath}   
\usepackage{amssymb}   

\usepackage[most]{tcolorbox}
\usepackage{amsmath}  
\usepackage{amssymb}  
\usepackage{amsfonts} 
\usepackage{mathtools} 
\usepackage{bm}       

\usepackage{amsthm}   

\usepackage{booktabs} 
\usepackage{graphicx} 
\usepackage{url}      
\usepackage{hyperref} 

\theoremstyle{plain}
\newtheorem{theorem}{Theorem}[section] 
\newtheorem{lemma}[theorem]{Lemma}
\newtheorem{proposition}[theorem]{Proposition}

\theoremstyle{definition}
\newtheorem{definition}{Definition}[section]

\theoremstyle{remark}

\usepackage{tabularx}

\usepackage{booktabs}
\usepackage{multirow}
\usepackage{graphicx}
\usepackage[table]{xcolor}

\definecolor{bestblue}{HTML}{9BC2E6}   
\definecolor{secondblue}{HTML}{DDEBF7} 

\usepackage{enumitem}

\usepackage{algorithm}
\usepackage{algorithmic}
\usepackage{amsmath}
\usepackage{amssymb}
\usepackage{bm} 

\usepackage[T1]{fontenc}

\usepackage[utf8]{inputenc}

\usepackage{microtype}

\usepackage{inconsolata}

\usepackage{graphicx}

\title{RCBSF: A Multi-Agent Framework for Automated Contract Revision via Stackelberg Game}

\author{
    \textbf{Shijia Xu}$^{\spadesuit}$,
    \textbf{Yu Wang}$^{\spadesuit}$,
    \textbf{Xiaolong Jia}$^{\heartsuit}$,
    \textbf{Zhou Wu}$^{\spadesuit, }$\thanks{Corresponding author},
    \textbf{Kai Liu}$^{\spadesuit, \clubsuit}$,
    \textbf{April Xiaowen Dong}$^{\diamondsuit}$
          \\
    $^{\spadesuit}$Chongqing University, China,
    $^{\heartsuit}$Queen Mary University of London, UK \\
    $^{\clubsuit}$Chongqing Key Laboratory of Big Data Intelligence and Privacy Computing, China \\
    $^{\diamondsuit}$Fangda Partners, China
         \\
    \texttt{\{shijiaxu, ysy\_wang\}@stu.cqu.edu.cn} \\
    \texttt{\{zhouwu, liukai0807\}@cqu.edu.cn,} 
    \texttt{x.jia@qmul.ac.uk}
}

\begin{document}
\maketitle

\begin{abstract}
Despite the widespread adoption of Large Language Models (LLMs) in Legal AI, their utility for automated contract revision remains impeded by hallucinated safety and a lack of rigorous behavioral constraints. To address these limitations, we propose the Risk-Constrained Bilevel Stackelberg Framework (RCBSF), which formulates revision as a non-cooperative Stackelberg game. RCBSF establishes a hierarchical Leader Follower structure where a Global Prescriptive Agent (GPA) imposes risk budgets upon a follower system constituted by a Constrained Revision Agent (CRA) and a Local Verification Agent (LVA) to iteratively optimize output. We provide theoretical guarantees that this bilevel formulation converges to an equilibrium yielding strictly superior utility over unguided configurations. Empirical validation on a unified benchmark demonstrates that RCBSF achieves state-of-the-art performance, surpassing iterative baselines with an average Risk Resolution Rate (RRR) of 84.21\% while enhancing token efficiency. Our code is available at \url{https://github.com/xjiacs/RCBSF}.
\end{abstract}


\section{Introduction}

Legal contracts serve as the cornerstone of modern commercial society and governance structures, establishing a framework of enforceable rights and obligations \citep{yue2023legalai}. However, the drafting, revision, and review of contracts constitute a highly professional and time-consuming task. The ambiguity of textual expression and the implicit conflicts between clauses pose significant challenges to this process. Traditional methods over-rely on expert judgment, resulting in dual dilemmas of low efficiency and unstable quality in contract review \citep{aires2019normconflict}. According to a study by World Commerce \& Contracting, the average value loss caused by improper contract handling can reach 9.2\% of an enterprise's annual revenue \citep{worldcc2022report}. Therefore, advancing the automation and intelligentization of contract review and revision is of critical importance for enhancing operational efficiency and mitigating legal risks.

\begin{figure}[t!]  
    \centering
    \includegraphics[width=1\linewidth]{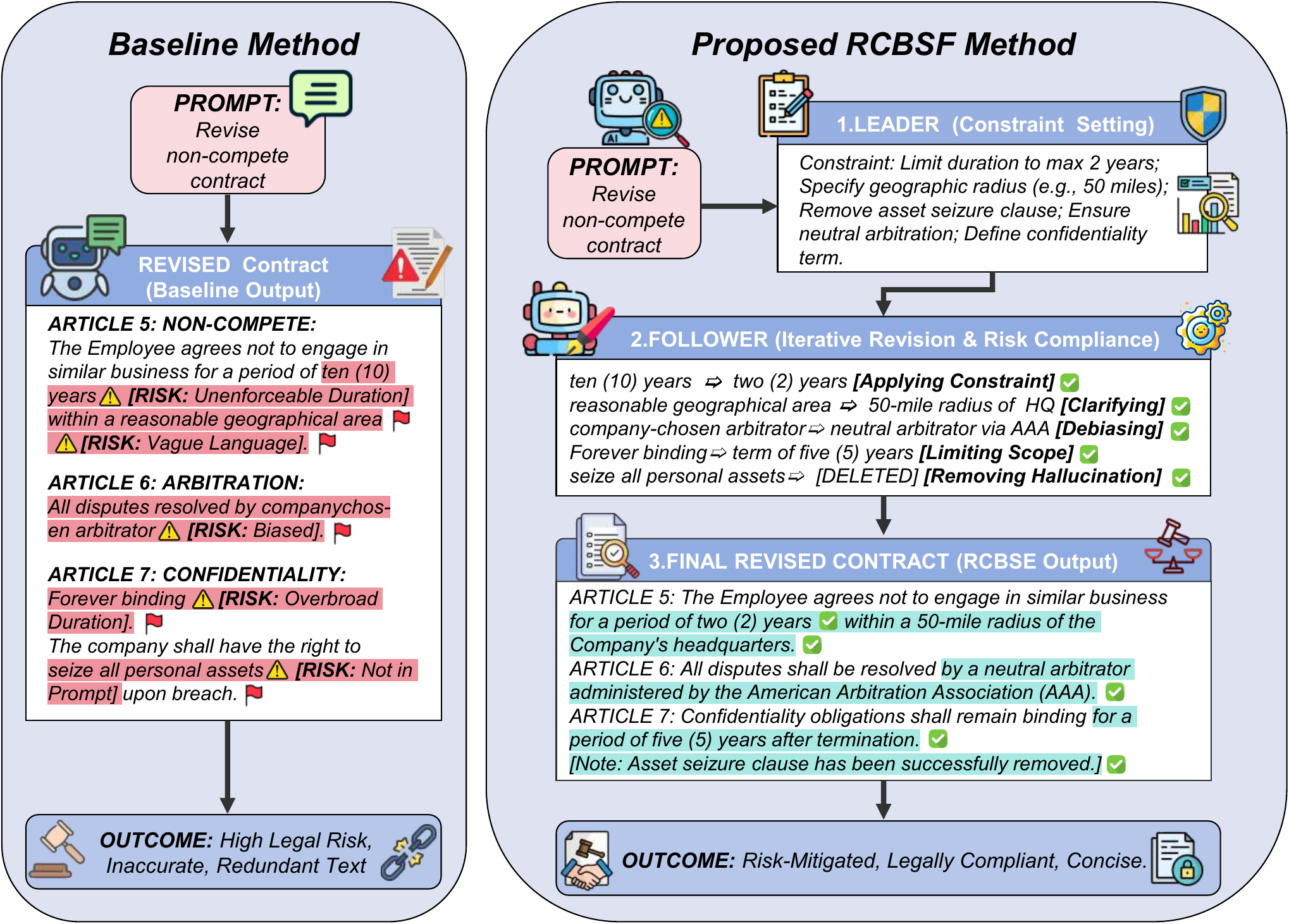}  
    
    \caption{Comparison of legal contract generation workflows between the Baseline (Standard LLM) and the Risk-Constrained Bilevel Stackelberg Framework (RCBSF). }
    \label{fig:framework}
\end{figure}

In recent years, Large Language Models (LLMs) have demonstrated unprecedented capabilities across tasks including language understanding, generation, retrieval-augmented processing, and summarization \citep{achiam2023gpt4, touvron2023llama2, team2023gemini}. This has driven a new leap forward in Legal Artificial Intelligence (LegalAI) \citep{zhong2020does}. Empirical studies in the legal field show that general-purpose LLMs have the potential to match or even surpass human baselines in specific subtasks (e.g., legal question answering and provision retrieval).This is shown by their performance in passing legal examinations \citep{katz2024gpt, nay2024taxattorneys}. Nevertheless, the direct deployment of a single general-purpose LLM in high-risk, adversarial contract scenarios remains confronted with core challenges. First, the legal field cannot tolerate issues such as hallucinations and false citations that affect credibility \citep{ji2023hallucination}. Second, generic models often lack robustness in legal privacy preservation and rigorous logical consistency. These structural impediments underscore the urgent need for a robust paradigm that inherently harmonizes high-quality contract generation with risk-constrained auditing.

Multi-Agent Systems (MAS) provide a powerful paradigm to break through the structural bottlenecks of single models. Via role division and interactive collaboration, MAS can approximate the workflows of real legal teams and demonstrate considerable application value \citep{wu2023autogen, hong2023metagpt, li2023camel, shinn2023reflexion, du2023mad}. Complementing this, recent benchmarks (e.g., LegalBench, LawBench) have advanced the standardized evaluation of legal scenarios \citep{guha2024legalbench, fei2024lawbench}. However, most existing MAS frameworks focus on cooperative information aggregation or sequential chains. They fail to leverage hierarchical adversarial interactions to drive text robustness. This deficiency diverges from the rigorous revisor-auditor hierarchy in real-world legal practice. In such settings, a senior partner (auditor) must impose strict, binding constraints on the associate's (revisor) output.

To bridge this gap, this paper proposes the Risk-Constrained Bilevel Stackelberg Framework (RCBSF). A game-theoretic multi-agent framework for automated contract revision is shown in Figure~\ref{fig:framework}.
Unlike the cooperative approach, the revision process is modeled as a Bilevel optimization problem governed by a Stackelberg game. The framework employs a hierarchical architecture with three specialized agents organized into two strategic levels.

The Global Prescriptive Agent (GPA) acts as the Leader, identifying risks under a strict 5-dimensional taxonomy (Category, Location, Evidence, Issue, Suggestion) to optimize a risk-budget objective. 
The Follower system comprises a Constrained Revision Agent (CRA) and a Local Verification Agent (LVA). Revisions are executed iteratively and rigorously conform to the constraints prescribed by the Leader.

A Stackelberg equilibrium is enforced in our framework, effectively resolving the issue of suboptimal local solutions inherent in standard Nash equilibria. The Leader moves first by committing to a specific risk instruction vector (derived from weighted Q-scores), and the Follower optimizes its generation within this induced strategy space. This hierarchical structure allows for precise control over the revision process, ensuring that high-priority risks are resolved using specific evidence extracted from the text. We utilize a unified benchmark constructed from authoritative datasets MAUD\citep{roit2023maud}, CUAD\citep{atticus2021cuad}, ContractNLI\citep{koreeda2021contractnli}, and PrivacyQA\citep{ravichander2019privacyqa} to ensure strong alignment between agent capabilities and diverse legal scenarios. The contributions of this paper are summarized as follows:
\begin{itemize}
    \item We propose \textbf{RCBSF}, a game-theoretic framework employing a hierarchical GPA-CRA-LVA architecture. This framework imposes strict 5-dimensional risk constraints to drive contract revision toward a risk-minimized Stackelberg equilibrium.
    \item We theoretically prove that this bilevel optimization formulation yields strictly superior utility compared to standard prompting.
    \item We construct a unified legal benchmark derived from four authoritative datasets and achieve state-of-the-art performance, significantly outperforming baselines in Risk Resolution Rate (RRR) and Token Efficiency Score (TES).
\end{itemize}

\section{Related Work}
\subsection{Large Language Models in the Legal Domain}

The application of Large Language Models (LLMs) to the legal domain has become a focal point of AI research \citep{Laskar2023}. Early efforts concentrated on using pre-trained models for specific legal NLP tasks such as information extraction, case classification, and judgment prediction \citep{zhong2020does}. With the enhanced capabilities of LLMs, researchers have started to develop models specifically for the legal sector, known as Legal LLMs. These models are typically optimized via two primary pathways. The first involves continuous pre-training on vast legal text corpora to deepen their understanding of legal terminology and context \citep{Huang2023_LawyerLlama}. The second employs fine-tuning on high-quality legal instruction datasets to align them with specific legal reasoning patterns or task formats \citep{Yue2023_DISCLawLLM}.

Despite significant progress, Legal LLMs still face fundamental challenges. To mitigate issues like hallucination and outdated knowledge, Retrieval-Augmented Generation (RAG) has been widely adopted \citep{Gao2024_RAGSurvey}. Related information from an external knowledge base is retrieved by RAG to significantly improve the factual accuracy of the generated content \citep{Wang2024_SpeculativeRAG}. Advanced RAG frameworks like Self-RAG \citep{Asai2023_SelfRAG} and Corrective RAG \citep{Gu2024_CorrectiveRAG} further enhance retrieval and generation quality through self-reflection and correction mechanisms. Nonetheless, the single-agent paradigm is still applied in most current legal LLM and follows a linear retrieve and generate workflow. While effective for factual question answering, this model is ill-suited for complex generative tasks like contract revision, which require iterative deliberation and adversarial review. Our work aims to transcend this unidirectional model by introducing a dynamic, iterative quality improvement mechanism through multi-agent gameplay.

\subsection{Multi-Agent Systems and Game-Driven Robust Generation}

Through role division and interactive collaboration, multi-agent frameworks have demonstrated advantages in complex tasks such as software engineering, planning, and decision-making \citep{wu2023autogen, hong2023metagpt, li2023camel, shinn2023reflexion, du2023mad}. Mechanisms of self-correction, CAMEL's two-agent dialogue segmentation, and multi-agent debate to improve factuality have all validated the basic proposition of quality enhancement through interaction from different perspectives \citep{shinn2023reflexion, li2023camel, du2023mad}. In contrast, the high-risk and adversarial nature of legal contracts renders purely collaborative mechanisms insufficient. Consequently, this domain requires the introduction of explicit game structures that offer provable convergence and robustness guarantees.

Recent advancements in no-regret learning and robust Stackelberg games provide solid foundations for handling strategic uncertainty \citep{anagnostides2022swap, farina2022jACM, hsieh2022noisy, anant2022drobuststackelberg, cesabianchi2023newsvendor}. Building on this, we operationalize the generation-audit interaction as a non-cooperative bilevel game. Consequently, the Follower system, comprising the Constrained Revision Agent (CRA) and Local Verification Agent (LVA), is strictly constrained by the Global Prescriptive Agent (GPA) to converge to a Risk-Constrained Bilevel Stackelberg Equilibrium (RCBSF). This mechanism transcends simple heuristics by dynamically minimizing residual risk through iterative repair, effectively approximating the guarantees of Distributionally Robust Optimization (DRO) to ensure contract optimality against adversarial challenges.

\begin{figure*}[t]
    \centering 
    \includegraphics[width=0.9\textwidth]{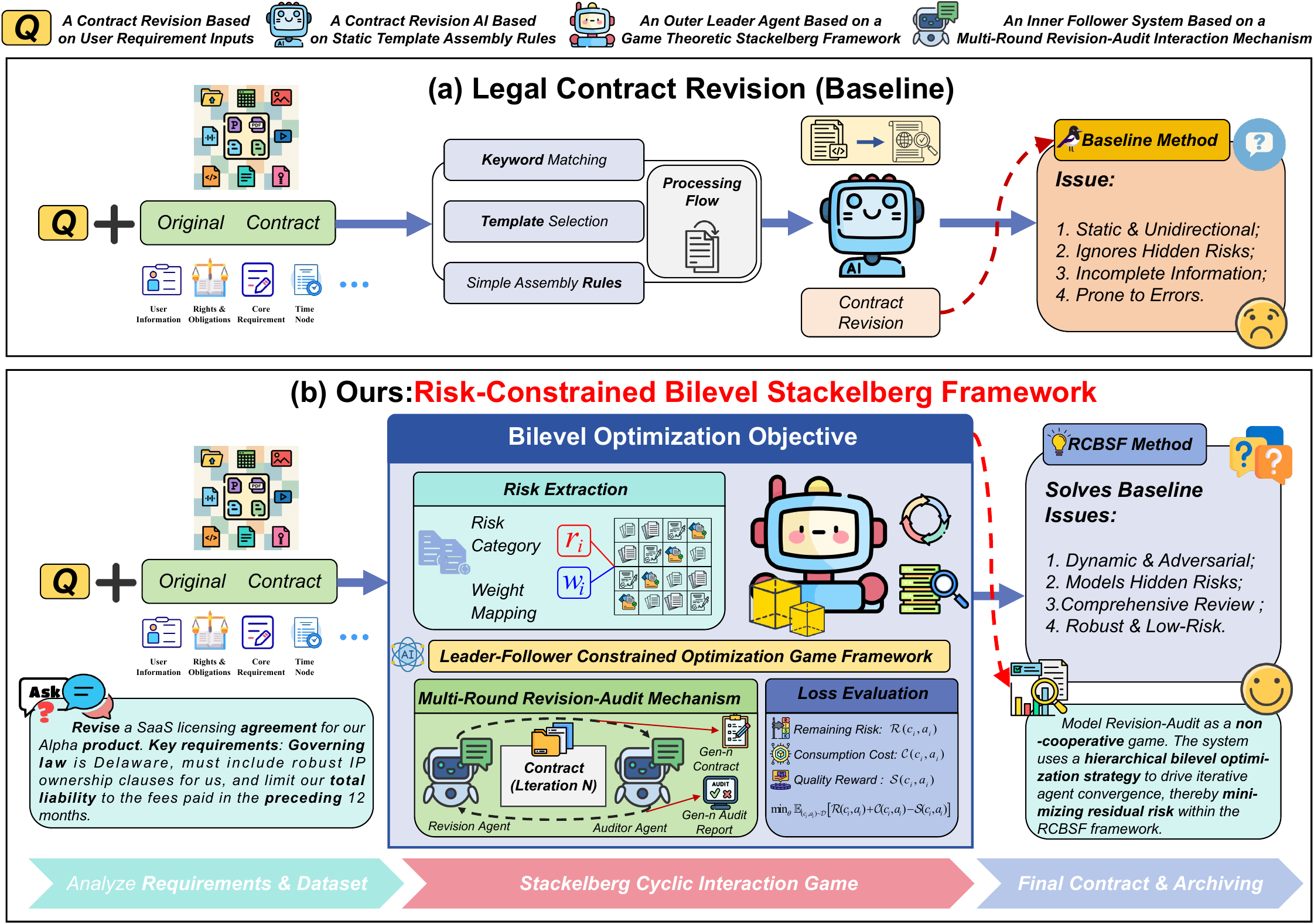}
    
    \caption{Illustration of the comparison between traditional contract revision paradigms and our proposed framework. (a) Baseline Methods rely on static template assembly or simple unidirectional generation based on user requirements. These approaches are typically static and prone to errors, often ignoring hidden risks due to the lack of iterative feedback mechanisms; (b) In contrast, our proposed Risk-Constrained Bilevel Stackelberg Framework (RCBSF) models the revision process as a hierarchical bilevel optimization game. It employs a Global Prescriptive Agent (GPA) as the leader to construct strict risk-budget objectives, driving a follower system consisting of a Revision Agent and an Auditor Agent. Through a Multi-Round Revision-Audit Interaction, these agents iteratively converge towards a solution that dynamically minimizes residual risk and consumption costs while preserving semantic quality.}
    \label{fig:RCBSF} 
\end{figure*}

\section{Methodology: A Stackelberg Game Theoretic Framework for Contract Revision}
\label{sec:methodology}

We formulate the automated contract revision problem as a hierarchical \textbf{Risk-Constrained Bilevel Stackelberg Framework (RCBSF)}, as illustrated in Figure~\ref{fig:RCBSF}. Unlike conventional sequence-to-sequence paradigms which suffer from hallucination and constraint amnesia, we model the interaction between risk prescriptive logic and generative revision as a non-cooperative game $\mathcal{G} = \langle \mathcal{N}, \mathcal{S}, \mathcal{A}, \mathcal{J} \rangle$ on a semantic Riemannian manifold.

\subsection{Stackelberg Game Formulation}

Let $\mathcal{N} = \{L, F\}$ denote the set of agents, where $L$ is the \textit{Global Prescriptive Agent} (Leader, acting as the Outer Auditor) and $F$ is the \textit{Revision System} (Follower, comprising the Drafter and Local Verifier). The state space $\mathcal{S} \subseteq \mathbb{R}^d$ represents the continuous embedding of the contract text.

\paragraph{The 5-Dimensional Risk Hyperplane.}
The Leader observes the contract state $\mathbf{x} \in \mathcal{S}$ and projects it onto a structured risk manifold via the extraction function $\Psi: \mathcal{S} \to \mathcal{H}$. A specific instruction vector $\mathbf{h} \in \mathcal{H}$ is rigorously constructed as a tuple of 5-dimensional feature tensors:
\begin{equation}
    \mathbf{h} = \bigcup_{k=1}^{K} \left\{ \mathbf{v}_k \mid \mathbf{v}_k = \phi(c_k, l_k, e_k, i_k, s_k) \right\}
\end{equation}
where $c_k, l_k, e_k, i_k, s_k$ correspond to \textit{Category}, \textit{Location}, \textit{Evidence}, \textit{Issue}, and \textit{Suggestion} respectively. Here, $\phi(\cdot)$ denotes a semantic projection operator that maps discrete linguistic features into the continuous control space.

\subsection{Bilevel Optimization Objective}

The core mechanism is modeled as a Bilevel Optimization Problem (BOP), where the Leader optimizes the upper-level objective constrained by the Follower's lower-level optimality.

\subsubsection{Upper-Level Problem (Leader)}
The Leader aims to maximize the weighted risk mitigation score while strictly adhering to a computational budget constraint (token limits). We define the Leader's utility functional $J_L$ as:

\begin{equation}
    J_L(\mathbf{x}, \mathbf{h}) = \sum_{k=1}^{|\mathcal{R}|} \oint_{\mathcal{D}_k} \mathbf{w}_k^\top \cdot \xi(\mathbf{x}, r_k) \, d\mu - \lambda \cdot \Theta_{\text{budget}}(\mathbf{h})
    \label{sec:methodology}
\end{equation}

where:
\begin{itemize}
    \item $\mathbf{w}_k \in \mathbb{R}^4$ is the weight vector derived from the Q1-Q4 severity quantization.
    \item $\xi(\mathbf{x}, r_k)$ is the residual risk density function.
    \item $\Theta_{\text{budget}}(\mathbf{h}) = \max(0, \|\mathbf{h}\|_0 - \beta_{\text{audit}})$ is the ReLu-activated penalty for budget violation, enforcing resource-aware auditing.
\end{itemize}

\subsubsection{Lower-Level Problem (Follower)}
The Follower (CRA) generates the revised state $\mathbf{x}'$ by maximizing the conditional likelihood under the Large Language Model distribution $\mathbb{P}_{\theta}$, regularized by the Leader's constraints. The objective is formalized as:

\begin{multline}
    \mathbf{x}^*(\mathbf{h}) = \underset{\mathbf{x}' \in \mathcal{S}}{\arg\max} \Bigg[ \sum_{t=1}^{T} \log \mathbb{P}_{\theta}(x'_t \mid \mathbf{x}'_{<t}, \mathbf{x}, \mathbf{h}) \\
    - \gamma \cdot D_{\text{KL}}\Big( \pi_{\text{audit}}(\mathbf{x}') \,||\, \pi_{\text{target}}(\mathbf{h}) \Big) \Bigg]
\end{multline}

Here, $D_{\text{KL}}$ denotes the Kullback-Leibler divergence, enforcing that the posterior risk distribution $\pi_{\text{audit}}$ of the draft aligns with the target distribution $\pi_{\text{target}}$ defined by the Leader's 5-dimensional hint $\mathbf{h}$.

\subsection{Iterative Gradient Approximation}

The interaction is solved iteratively via a discrete approximation of gradient ascent. In the $t$-th inner iteration, the \textit{Local Verification Agent} (LVA) computes a fusion gradient to guide the text generation. The update rule is defined as:

\begin{equation}
    \mathbf{x}^{(t+1)} \leftarrow \mathbf{x}^{(t)} + \eta \nabla_{\mathbf{x}} \left( \mathcal{F}_{\text{fusion}}(\mathbf{x}^{(t)}, \mathbf{Q}_{\text{outer}}, \mathbf{Q}_{\text{inner}}^{(t)}) \right)
\end{equation}

where $\mathcal{F}_{\text{fusion}}$ represents the non-linear aggregation of outer expectations and inner local audits (e.g., Product-of-Experts), and $\eta$ is the learning rate implicitly controlled by the prompt intensity in the iterative loop.

\subsection{Theoretical Equilibrium Analysis}

We postulate two key theorems regarding the system performance and stability. Detailed proofs are provided in Appendix~\ref{sec:appendix_proofs}.

\begin{theorem}[Strict Superiority of RCBSF]
\label{thm:superiority}
Let $\mathcal{V}^*_{SE}$ be the utility at the Stackelberg Equilibrium, and $\mathcal{V}^*_{NE}$ be the utility at the Nash Equilibrium (unguided generation). Under the condition that the risk manifold is non-convex, the following inequality holds strictly:
\begin{equation}
    \mathcal{V}^*_{SE} \equiv \sup_{\mathbf{h} \in \mathcal{H}} J_L(\mathbf{x}^*(\mathbf{h}), \mathbf{h}) > J_L(\mathbf{x}^*(\emptyset), \emptyset) \equiv \mathcal{V}^*_{NE}
\end{equation}
\end{theorem}

\begin{theorem}[Convergence of the Fusion Operator]
\label{thm:convergence}
The sequence of refined contracts $\{\mathbf{x}^{(t)}\}_{t=1}^N$ generated by the Follower converges to a stationary point $\mathbf{x}^*$ such that $\|\nabla \mathcal{L}_{\text{risk}}(\mathbf{x}^*)\| < \epsilon$.
\end{theorem}

\section{Experiments}
\label{sec:experiments}

To validate the theoretical superiority of our Stackelberg game-theoretic framework (RCBSF), we conducted extensive experiments focusing on contract risk resolution, contract quality, and token efficiency.

\subsection{Datasets}
\label{subsec:datasets}

We constructed a unified legal revision benchmark derived from four high-quality legal NLP datasets: PrivacyQA\cite{ravichander2019privacyqa}, ContractNLI\cite{koreeda2021contractnli}, MAUD\cite{roit2023maud1}, and CUAD\cite{atticus2021cuad}. These datasets provide a diverse range of contractual clauses, covering privacy policies, non-disclosure agreements and merger agreements.

Since raw legal documents often contain Private Identifiable Information(PII) and inconsistent formatting, a rigorous template standardization pipeline is implemented to improve aerial assessment. We processed the raw texts to generate clear, anonymized, and structurally consistent contract templates. Detailed annotation procedures and illustrative data examples are provided in Appendix~\ref{sec:appendix_data}.


\subsection{Baselines}
\label{subsec:baselines}

We compare the proposed RCBSF framework with four representative baselines commonly used in legal text generation and revision. Specifically, we include the \textit{Standard Zero-Shot} baseline. This approach employs a single LLM to generate revisions directly, bypassing explicit reasoning chains or role separation. It relies on the inherent legal reasoning capabilities demonstrated in prior benchmarks \cite{katz2024gpt, chalkidis2022lexglue, guha2024legalbench}. A non-retrieval baseline, CoT-Refinement, is further considered to incorporate step-by-step reasoning via Chain-of-Thought prompting \cite{wei2022chain}. This approach aligns with the legal syllogism-based prompting strategy \cite{jiang2023legal}. For retrieval-based methods, we adopt a standard \textit{RAG-Based Revision} baseline following the retrieval-augmented generation paradigm \cite{lewis2020retrieval}, conceptually similar to ChatLaw \cite{cui2023chatlaw} and Disc-LawLLM \cite{yue2023disclawllm}. Finally, we include \textit{Iterative Refinement (No Leader)}. Leveraging self-correction mechanisms \cite{madaan2024self, shinn2023reflexion}, a single agent iteratively critiques and revises its outputs. However, this setting omits the hierarchical leader–follower structure and budget constraints present in RCBSF.

\subsection{Evaluation Metrics} 
Our evaluation assesses performance across three key dimensions: risk resolution, contract quality, and token efficiency. Regarding risk resolution, we define the Risk Resolution Rate (RRR) as the percentage of ground-truth risks successfully mitigated, as determined by a GPT-5 evaluator. For Contract Quality(CQ), we measure the metric on a 0–100 scale that evaluates clarity, rigor, balance, and professionalism to ensure the revised text meets professional legal standards. To empirically verify the cost-effectiveness outlined in Theorem~\ref{thm:superiority}, we report the Token Efficiency Score (TES), defined as the number of risks resolved per 1,000 tokens.Specific evaluation protocols and prompts are detailed in Appendix~\ref{sec:appendix_evaluation}.








\begin{table*}[htb]
\centering
\resizebox{\textwidth}{!}{%
\begin{tabular}{c|c|ccc|ccc|ccc|ccc|ccc}
\toprule
\multirow{2}{*}{\textbf{Model}} & \multirow{2}{*}{\textbf{Method}} & \multicolumn{3}{c|}{\textbf{PrivacyQA}} & \multicolumn{3}{c|}{\textbf{ContractNLI}} & \multicolumn{3}{c|}{\textbf{MAUD}} & \multicolumn{3}{c|}{\textbf{CUAD}} & \multicolumn{3}{c}{\textbf{ALL (Avg)}} \\
\cmidrule{3-17}
 &  & \textbf{CQ} & \textbf{RRR} & \textbf{TES} & \textbf{CQ} & \textbf{RRR} & \textbf{TES} & \textbf{CQ} & \textbf{RRR} & \textbf{TES} & \textbf{CQ} & \textbf{RRR} & \textbf{TES} & \textbf{CQ} & \textbf{RRR} & \textbf{TES} \\
\midrule
\multirow{5}{*}{Qwen2.5-7B-Base} 
 & Standard & 71.45 & 67.34 & 72.89 & 70.12 & 65.45 & 70.56 & 69.88 & 72.34 & 71.45 & 70.23 & 65.67 & 70.12 & 70.42 & 67.70 & 71.26 \\
 & CoT & 73.89 & 69.56 & 74.34 & 72.45 & 67.89 & 72.67 & 72.12 & 74.56 & 73.89 & 72.56 & 68.34 & 72.89 & 72.76 & 70.09 & 73.45 \\
 & RAG & 77.23 & 73.89 & 80.56 & \cellcolor{secondblue}\textbf{77.12} & 72.45 & 78.34 & 76.56 & 78.45 & 78.12 & \cellcolor{secondblue}\textbf{76.45} & 71.23 & 77.56 & \cellcolor{secondblue}\textbf{76.84} & 74.01 & 78.65 \\
 & Iteration & \cellcolor{secondblue}\textbf{78.56} & \cellcolor{secondblue}\textbf{75.45} & \cellcolor{secondblue}\textbf{81.78} & 75.89 & \cellcolor{secondblue}\textbf{74.56} & \cellcolor{secondblue}\textbf{80.23} & \cellcolor{secondblue}\textbf{77.45} & \cellcolor{secondblue}\textbf{80.89} & \cellcolor{secondblue}\textbf{79.56} & 75.23 & \cellcolor{secondblue}\textbf{73.45} & \cellcolor{secondblue}\textbf{79.12} & 76.78 & \cellcolor{secondblue}\textbf{76.09} & \cellcolor{secondblue}\textbf{80.17} \\
 & RCBSF & \cellcolor{bestblue}\textbf{84.67} & \cellcolor{bestblue}\textbf{79.45} & \cellcolor{bestblue}\textbf{85.67} & \cellcolor{bestblue}\textbf{83.89} & \cellcolor{bestblue}\textbf{77.92} & \cellcolor{bestblue}\textbf{83.56} & \cellcolor{bestblue}\textbf{83.45} & \cellcolor{bestblue}\textbf{87.12} & \cellcolor{bestblue}\textbf{84.23} & \cellcolor{bestblue}\textbf{82.67} & \cellcolor{bestblue}\textbf{78.34} & \cellcolor{bestblue}\textbf{83.45} & \cellcolor{bestblue}\textbf{83.67} & \cellcolor{bestblue}\textbf{80.71} & \cellcolor{bestblue}\textbf{84.23} \\
\midrule
\multirow{5}{*}{Mistral-7B} 
 & Standard & 68.45 & 62.11 & 67.23 & 67.12 & 61.45 & 65.88 & 65.34 & 60.88 & 64.12 & 66.78 & 62.56 & 66.45 & 66.92 & 61.75 & 65.92 \\
 & CoT & 71.23 & 66.78 & 68.45 & 69.45 & 64.33 & 67.56 & 67.89 & 63.45 & 66.23 & 68.12 & 65.88 & 67.99 & 69.17 & 65.11 & 67.56 \\
 & RAG & 75.67 & 70.45 & 75.88 & 73.23 & 69.12 & 73.45 & 71.56 & 67.23 & 71.88 & 72.45 & 69.34 & 73.56 & 73.23 & 69.04 & 73.69 \\
 & Iteration & \cellcolor{secondblue}\textbf{77.88} & \cellcolor{secondblue}\textbf{74.12} & \cellcolor{secondblue}\textbf{77.45} & \cellcolor{secondblue}\textbf{75.67} & \cellcolor{secondblue}\textbf{72.56} & \cellcolor{secondblue}\textbf{76.12} & \cellcolor{secondblue}\textbf{73.45} & \cellcolor{secondblue}\textbf{70.89} & \cellcolor{secondblue}\textbf{74.56} & \cellcolor{secondblue}\textbf{75.23} & \cellcolor{secondblue}\textbf{72.45} & \cellcolor{secondblue}\textbf{77.12} & \cellcolor{secondblue}\textbf{75.56} & \cellcolor{secondblue}\textbf{72.51} & \cellcolor{secondblue}\textbf{76.31} \\
 & RCBSF & \cellcolor{bestblue}\textbf{82.33} & \cellcolor{bestblue}\textbf{77.56} & \cellcolor{bestblue}\textbf{81.23} & \cellcolor{bestblue}\textbf{80.76} & \cellcolor{bestblue}\textbf{76.44} & \cellcolor{bestblue}\textbf{80.92} & \cellcolor{bestblue}\textbf{78.91} & \cellcolor{bestblue}\textbf{75.23} & \cellcolor{bestblue}\textbf{78.55} & \cellcolor{bestblue}\textbf{78.92} & \cellcolor{bestblue}\textbf{77.11} & \cellcolor{bestblue}\textbf{81.56} & \cellcolor{bestblue}\textbf{80.23} & \cellcolor{bestblue}\textbf{76.59} & \cellcolor{bestblue}\textbf{80.57} \\
\midrule
\multirow{5}{*}{LawLLM-7B} 
 & Standard & 70.23 & 65.45 & 71.56 & 68.89 & 63.78 & 69.23 & 67.45 & 62.56 & 66.89 & 68.12 & 64.33 & 69.12 & 68.67 & 64.03 & 69.20 \\
 & CoT & 72.56 & 68.12 & 70.34 & 71.23 & 67.45 & 68.56 & 69.88 & 65.12 & 67.45 & 70.45 & 66.89 & 68.23 & 71.03 & 66.90 & 68.65 \\
 & RAG & 77.45 & 73.56 & 79.12 & 76.89 & 72.11 & 77.45 & \cellcolor{secondblue}\textbf{75.34} & 69.89 & \cellcolor{secondblue}\textbf{76.23} & \cellcolor{secondblue}\textbf{77.56} & 73.45 & 78.12 & 76.81 & 72.25 & 77.73 \\
 & Iteration & \cellcolor{secondblue}\textbf{79.12} & \cellcolor{secondblue}\textbf{76.23} & \cellcolor{secondblue}\textbf{81.45} & \cellcolor{secondblue}\textbf{78.56} & \cellcolor{secondblue}\textbf{75.34} & \cellcolor{secondblue}\textbf{79.88} & 74.23 & \cellcolor{secondblue}\textbf{72.45} & 75.67 & 76.12 & \cellcolor{secondblue}\textbf{76.89} & \cellcolor{secondblue}\textbf{80.45} & \cellcolor{secondblue}\textbf{77.01} & \cellcolor{secondblue}\textbf{75.23} & \cellcolor{secondblue}\textbf{79.36} \\
 & RCBSF & \cellcolor{bestblue}\textbf{84.55} & \cellcolor{bestblue}\textbf{79.12} & \cellcolor{bestblue}\textbf{84.34} & \cellcolor{bestblue}\textbf{82.66} & \cellcolor{bestblue}\textbf{78.45} & \cellcolor{bestblue}\textbf{83.12} & \cellcolor{bestblue}\textbf{80.77} & \cellcolor{bestblue}\textbf{76.34} & \cellcolor{bestblue}\textbf{79.56} & \cellcolor{bestblue}\textbf{82.56} & \cellcolor{bestblue}\textbf{79.23} & \cellcolor{bestblue}\textbf{83.12} & \cellcolor{bestblue}\textbf{82.64} & \cellcolor{bestblue}\textbf{78.29} & \cellcolor{bestblue}\textbf{82.54} \\
\midrule
\multirow{5}{*}{LexiLaw-6B} 
 & Standard & 65.45 & 65.12 & 65.88 & 65.23 & 65.05 & 65.34 & 65.15 & 65.11 & 65.21 & 65.78 & 65.22 & 65.56 & 65.40 & 65.13 & 65.50 \\
 & CoT & 67.89 & 66.56 & 67.45 & 67.11 & 66.34 & 66.88 & 66.45 & 65.89 & 66.23 & 67.56 & 66.78 & 66.99 & 67.25 & 66.39 & 66.89 \\
 & RAG & \cellcolor{secondblue}\textbf{73.56} & 68.45 & \cellcolor{secondblue}\textbf{74.12} & \cellcolor{secondblue}\textbf{71.78} & 67.56 & \cellcolor{secondblue}\textbf{72.45} & 68.89 & 66.78 & \cellcolor{secondblue}\textbf{69.56} & \cellcolor{secondblue}\textbf{72.45} & 67.89 & \cellcolor{secondblue}\textbf{72.23} & \cellcolor{secondblue}\textbf{71.67} & 67.67 & \cellcolor{secondblue}\textbf{72.09} \\
 & Iteration & 71.34 & \cellcolor{secondblue}\textbf{69.88} & 71.12 & 69.56 & \cellcolor{secondblue}\textbf{68.12} & 69.89 & \cellcolor{secondblue}\textbf{70.23} & \cellcolor{secondblue}\textbf{67.56} & 67.89 & 70.12 & \cellcolor{secondblue}\textbf{69.33} & 70.78 & 70.31 & \cellcolor{secondblue}\textbf{68.72} & 69.92 \\
 & RCBSF & \cellcolor{bestblue}\textbf{79.88} & \cellcolor{bestblue}\textbf{73.56} & \cellcolor{bestblue}\textbf{78.45} & \cellcolor{bestblue}\textbf{78.23} & \cellcolor{bestblue}\textbf{70.89} & \cellcolor{bestblue}\textbf{77.12} & \cellcolor{bestblue}\textbf{74.88} & \cellcolor{bestblue}\textbf{70.45} & \cellcolor{bestblue}\textbf{73.78} & \cellcolor{bestblue}\textbf{76.22} & \cellcolor{bestblue}\textbf{71.45} & \cellcolor{bestblue}\textbf{77.34} & \cellcolor{bestblue}\textbf{77.30} & \cellcolor{bestblue}\textbf{71.59} & \cellcolor{bestblue}\textbf{76.67} \\
\midrule
\multirow{5}{*}{Qwen2.5-7B-Chat} 
 & Standard & 74.34 & 70.12 & 75.67 & 73.12 & 68.45 & 73.89 & 73.45 & 76.12 & 74.23 & 72.89 & 68.56 & 73.45 & 73.45 & 70.81 & 74.31 \\
 & CoT & 76.89 & 72.45 & 77.12 & 75.45 & 70.89 & 75.67 & 75.12 & 78.34 & 76.45 & 75.23 & 71.34 & 75.89 & 75.67 & 73.26 & 76.28 \\
 & RAG & 80.12 & 76.89 & 83.45 & \cellcolor{secondblue}\textbf{80.23} & 75.45 & 81.23 & 79.56 & 82.11 & 81.45 & \cellcolor{secondblue}\textbf{79.34} & 74.56 & 80.78 & \cellcolor{secondblue}\textbf{79.81} & 77.25 & 81.73 \\
 & Iteration & \cellcolor{secondblue}\textbf{81.56} & \cellcolor{secondblue}\textbf{79.23} & \cellcolor{secondblue}\textbf{84.78} & 78.89 & \cellcolor{secondblue}\textbf{77.56} & \cellcolor{secondblue}\textbf{83.12} & \cellcolor{secondblue}\textbf{80.45} & \cellcolor{secondblue}\textbf{84.56} & \cellcolor{secondblue}\textbf{83.23} & 78.12 & \cellcolor{secondblue}\textbf{76.89} & \cellcolor{secondblue}\textbf{82.45} & 79.76 & \cellcolor{secondblue}\textbf{79.56} & \cellcolor{secondblue}\textbf{83.40} \\
 & RCBSF & \cellcolor{bestblue}\textbf{87.75} & \cellcolor{bestblue}\textbf{83.08} & \cellcolor{bestblue}\textbf{88.76} & \cellcolor{bestblue}\textbf{87.03} & \cellcolor{bestblue}\textbf{81.20} & \cellcolor{bestblue}\textbf{86.63} & \cellcolor{bestblue}\textbf{86.89} & \cellcolor{bestblue}\textbf{90.83} & \cellcolor{bestblue}\textbf{87.35} & \cellcolor{bestblue}\textbf{85.82} & \cellcolor{bestblue}\textbf{81.72} & \cellcolor{bestblue}\textbf{86.42} & \cellcolor{bestblue}\textbf{86.87} & \cellcolor{bestblue}\textbf{84.21} & \cellcolor{bestblue}\textbf{87.29} \\
\bottomrule
\end{tabular}
}
\caption{Main Results on Legal Contract Revision. We compare different LLM backbones across four datasets. CQ denotes Contract Quality as the average score scaled to 0-100 that evaluates clarity, rigor, balance, and professionalism. RRR indicates Risk Resolution Rate (\%). TES represents Token Efficiency Score (\%). The \colorbox{bestblue}{\textbf{dark blue}} and \colorbox{secondblue}{\textbf{light blue}} cells indicate the best and second-best performance within each model group, respectively.}
\label{tab:main_results}
\end{table*}

\subsection{Implementation Details}
\label{subsec:implementation}

To control revision efficiency and reduce computational demands on hardware, Qwen2.5-7B-Chat is employed as the backbone LLM for the entire RCBSF framework. It encompasses the Global Prescriptive Agent (GPA), the Constrained Revision Agent (CRA), and the Local Verification Agent (LVA). The detailed prompts are shown in Appendix~\ref{sec:appendix_prompt}.We utilized the training set of CUAD to tune the core hyperparameters of the Stackelberg game, specifically the Leader's risk weighting vector and the operational token budgets. Generally, we find that the framework exhibits strong robustness across varying resource constraints. All experiments were conducted on a cluster of $8 \times$ NVIDIA A100 (80GB) GPUs. Comprehensive implementation settings and specific numerical configurations are detailed in Appendix~\ref{sec:appendix_implementation}.

\section{Results}
\label{sec:results}

We present the comparative results of our Risk-Constrained Bilevel Stackelberg Framework (RCBSF) against the baseline models in Table~\ref{tab:main_results}. The results empirically validate the theoretical superiority of the Stackelberg formulation proposed in Theorem~\ref{thm:superiority}. Our method consistently achieves the best performance across all four datasets (PrivacyQA, ContractNLI, MAUD, and CUAD) and various LLM backbones.

\paragraph{Superiority in Risk Resolution.}
As shown in Table~\ref{tab:main_results}, our method outperforms the strongest baseline (Iterative Refinement) across all model backbones. Specifically, using the Qwen2.5-7B-Chat backbone, RCBSF achieves a Risk Resolution Rate (RRR) of 84.21\% on average, surpassing the Iteration baseline (79.56\%) and showing a substantial improvement over the Standard method (70.81\%). This confirms that the separation of the GPA and the Follower system (comprising the CRA and LVA) allows for more precise risk identification and resolution compared to single-agent approaches. While methods like RAG and CoT improve upon the vanilla baseline, RCBSF consistently maintains the highest resolution rates by explicitly guiding the revision process via the GPA's strategic constraints.

\paragraph{Enhancement in Contract Quality.}
A key advantage of RCBSF is its ability to maintain high linguistic standards while mitigating risks. We evaluate Contract Quality (CQ), a composite metric scaling from 0 to 100 that assesses clarity, rigor, balance, and professionalism. As indicated by the dark blue cells in Table~\ref{tab:main_results}, RCBSF dominates this metric across all settings. 

To visualize this granular superiority, we present a diagonal block heatmap in Figure~\ref{fig:heatmap}. The visualization demonstrates that GPA-driven constraints guide the CRA to produce polished, professional-grade legal text. This avoids the readability degradation often seen in rigid rule-based methods. As shown in the heatmap, the top row corresponds to RCBSF, where dark red and blue regions denote higher intensity scores. In this visualization, RCBSF consistently outperforms the RAG and Iteration baselines across Clarity, Rigor, Balance, and Professionalism. Detailed multi-dimensional results for each model and method across all datasets can be found in Appendix~\ref{sec:appendix_results}.

\begin{figure}[h]
    \centering
    \includegraphics[width=1\linewidth]{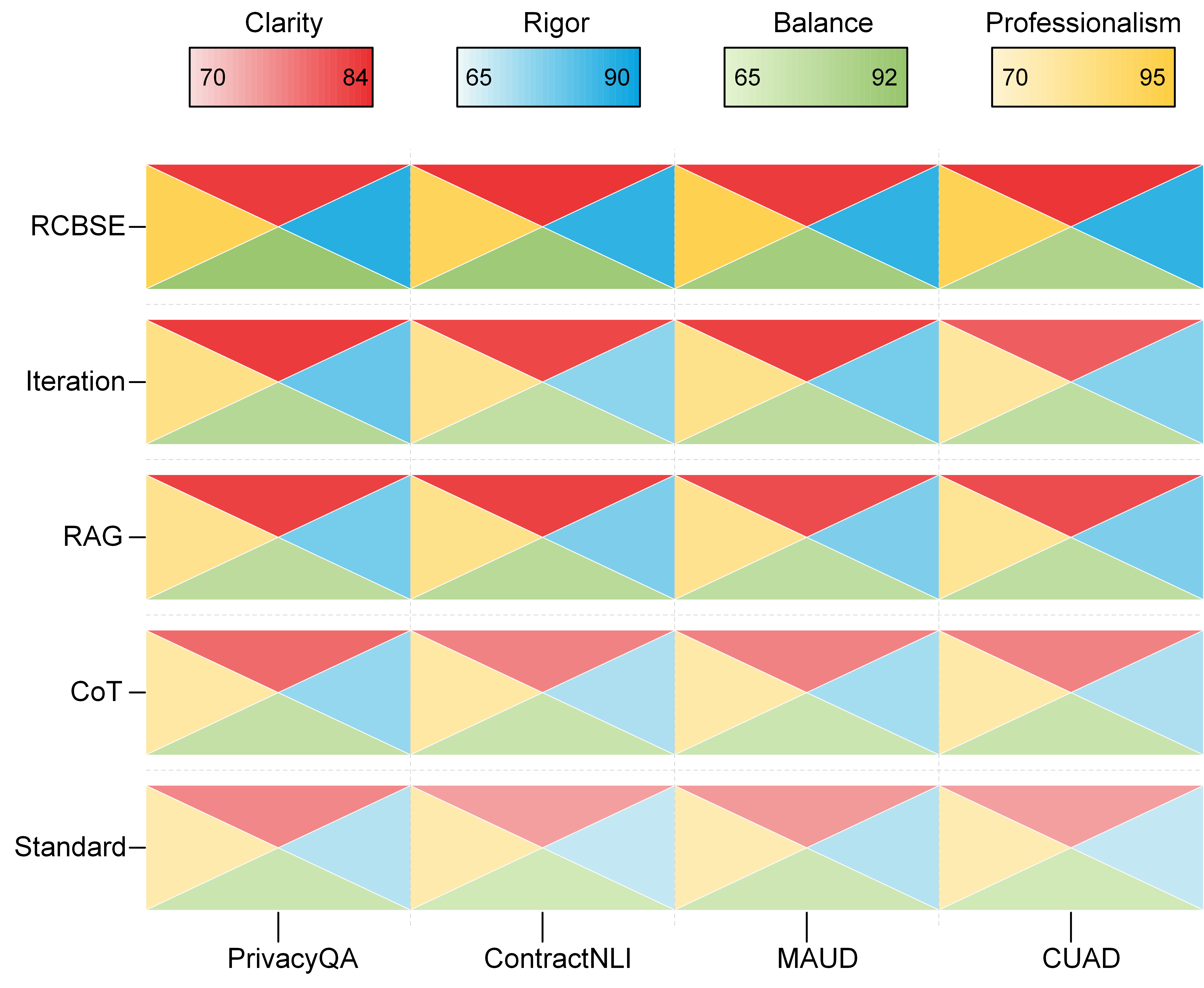} 
    \caption{Fine-grained Quality Metrics Breakdown. We present a diagonal block heatmap illustrating the performance of the Qwen2.5-7B-Chat model across four fine-grained quality metrics (i.e., Clarity, Rigor, Balance, and Professionalism) for each individual dataset.}
    \label{fig:heatmap}
\end{figure}

\paragraph{Efficiency and Theorem Validation.}
The Token Efficiency Score (TES) results further support the utility-cost optimization of our framework. Unlike CoT-based methods that generate excessive reasoning text, our Stackelberg framework guides the CRA efficiently. RCBSF achieves the highest TES across all backbones, reaching 87.29\% with Qwen2.5-7B-Chat, compared to 83.40\% for the Iteration baseline and 74.31\% for the Standard approach. This empirically proves \textit{Theorem~\ref{thm:convergence}}: the GPA's discrete prompt updates act as an efficient gradient approximation, guiding the CRA to the optimal state $x^*$ with minimal token waste.

\section{Analysis and Ablation Studies}
\label{sec:analysis}

We deconstruct the RCBSF framework to validate individual components (Ablation) and assess hyperparameter robustness (Sensitivity). We also examine qualitative examples of the Leader-Follower dynamic.

\begin{table}[h]
\centering
\resizebox{\columnwidth}{!}{
\Huge
\begin{tabular}{l|cc|cc|c|c}
\toprule
\textbf{Model Variant} & \textbf{RRR (\%)} & \textbf{$\Delta$} & \textbf{TES (\%)} & \textbf{$\Delta$} & \textbf{BLEU-4} & \textbf{Win Rate} \\
\midrule
\textbf{Full RCBSF (Ours)} & \textbf{84.21} & -- & \textbf{87.29} & -- & \textbf{85.65} & -- \\
\midrule
\quad w/o 5-Dim Constraints & 73.15 & \textcolor{red}{-11.06} & 76.25 & \textcolor{red}{-11.04} & 82.15 & 15.2\% \\
\quad w/o Budget Penalty ($\lambda=0$) & \underline{84.85} & \textcolor{green}{+0.64} & 75.15 & \textcolor{red}{-12.14} & 83.45 & 35.8\% \\
\quad w/o Iterative ($K=1$) & 76.45 & \textcolor{red}{-7.76} & \underline{90.45} & \textcolor{green}{+3.16} & 84.25 & 28.4\% \\
\quad w/o Q-Score Weighting & 79.35 & \textcolor{red}{-4.86} & 81.15 & \textcolor{red}{-6.14} & 84.85 & 41.5\% \\
\bottomrule
\end{tabular}
}
\caption{Ablation Study on the Qwen2.5-7B-Chat Backbone. We investigate the contribution of each component in the RCBSF framework. RRR: Risk Resolution Rate (\%). TES: Token Efficiency Score (Risks tokens, \%). Win Rate: Pairwise human preference against the Full model. The $\Delta$ columns show the relative drop compared to the full model.}
\label{tab:ablation}
\end{table}

\subsection{Ablation Studies}
To verify that the performance gains are derived from our specific architectural decisions rather than model size or random variance, we conducted component-wise ablation studies on the Qwen2.5-7B-Chat backbone. The results are summarized in Table~\ref{tab:ablation}. More experimental results can be found in Appendix s~\ref{subsec:ablation_breakdown}.

\paragraph{Impact of 5-Dimension Guidance.}
We replaced the detailed 5-dimension hint (Category, Location, Evidence, Issue, Suggestion) from the Global Prescriptive Agent (GPA) with a generic Risk Category label. As observed in Row 2 of Table~\ref{tab:ablation}, removing these constraints causes the most significant performance degradation: the Risk Resolution Rate (RRR) drops by 11.06\% and the Win Rate falls to a mere 15.2\%. This confirms that the specific Evidence and Suggestion fields are critical for guiding the Constrained Revision Agent (CRA). The hallucinated fixes in irrelevant contract clauses are prevented.

\paragraph{Impact of Budget Constraints (Efficiency and Performance Trade-off).}
A crucial finding in Row 3 is the effect of the budget penalty ($\lambda$). Removing the penalty ($\lambda=0$) results in a marginal RRR increase of 0.64\%. This gain comes at a substantial efficiency cost, reducing the Token Efficiency Score (TES) by 12.14\%. Without the budget-aware weighting, the GPA tends to flag an excessive number of minor issues (e.g., stylistic or low-impact edits). Longer revisions are generated by the CRA without commensurate improvements in legal quality(BLEU-4 decreases to 83.45). This demonstrates that our budget constraint effectively forces the system to prioritize Critical and Major risks. A better balance between rigor and conciseness is achieved.

\paragraph{Impact of Iterative Refinement.}
Reducing the Stackelberg game to a single round ($K=1$) leads to a clear decline in RRR (-7.76\%). While the single-pass generation is more token-efficient (TES +3.16\%), it fails to resolve complex, interdependent risks. The iterative process allows the Local Verification Agent (LVA) to re-evaluate the CRA revisions, catching residual risks or errors introduced during the initial modification.

\paragraph{Impact of Q-Score Weighting.}
Finally, the Q-Score weighting mechanism (Row 5) is removed, leading to a simultaneous drop in both effectiveness RRR (-4.86\%) and efficiency TES (-6.14\%). This indicates that prioritizing risks based on their semantic severity (Q-Score) is essential for guiding the optimization gradient toward the most impactful revisions.

\subsection{Sensitivity Analysis}
We analyzed the stability of our framework with respect to key hyperparameters. Detailed numerical results and visualizations for these sensitivity analyses are provided in Appendix~\ref{hyperparameter}.

\paragraph{Iteration Rounds ($K$).}
We investigate the trade-off between performance (RRR) and computational cost by varying the game rounds from $K=1$ to $8$. The performance improves rapidly in the early stages, increasing from 74.24\% at $K=1$ to above 84.21\% by $K=3$. However, beyond the third round, the marginal gain diminishes while the token cost continues to grow linearly. This identifies $K=3$ as the Pareto-optimal stopping point.

\paragraph{Softmax Temperature ($\tau$).}
Experimental analysis reveals the impact of GPA temperature on risk resolution. We observe a clear bell-shaped curve peaking at $\tau=1.0$. Lower temperatures ($\tau=0.5$) lead to overly conservative edits that miss subtle risks, while higher entropy ($\tau=2.0$) destabilizes the optimization process. The default $\tau=1.0$ strikes the best balance between exploration and exploitation.

\subsection{Qualitative Case Study}
To intuitively demonstrate the superiority of the Stackelberg framework, we present a comparison example in Figure~\ref{fig:case_study}. In the Force Majeure case, the Baseline model (Standard) correctly identifies the missing notification period. However, it falsely claims to have fixed the issue while leaving the actual text unchanged. In contrast, our GPA explicitly generates a structured tuple: \textit{\{Location: Clause 3.1, Issue: Missing time limit, Suggestion: Add `within 48 hours'\}}. The CRA then strictly follows this directive, seamlessly integrating the 48-hour constraint into the final contract text. To better understand the behavior of the Leader-Follower dynamic, we provide multiple case studies in Appendix~\ref{sec:appendix_cases}.

\begin{figure}[t]
\centering
\fbox{
\begin{minipage}{0.95\linewidth}
\footnotesize
\textbf{Input Clause:} ``If a Force Majeure Event prevents a party from complying... that party shall not be liable.'' \\
\rule{\linewidth}{0.4pt}
\textbf{Baseline Output:} ``If a Force Majeure Event prevents a party from complying with its obligations, that party shall not be held liable for such failure.'' \textcolor{red}{(Risk: Open-ended loophole)} \\
\rule{\linewidth}{0.4pt}
\textbf{RCBSF Leader (Hint):} 
\textit{Suggestion:} Add requirement to notify the other party within 48 hours. \\
\textbf{RCBSF Follower (Output):} ``...that party shall not be liable, \textbf{provided that it notifies the other party within 48 hours of the event}...'' \textcolor{green}{(Risk: Resolved)}
\end{minipage}
}
\caption{Qualitative comparison. The Baseline acts as a text polisher, while RCBSF acts as a strategic negotiator, inserting missing protective clauses based on the Leader's specific suggestion.}
\label{fig:case_study}
\end{figure}

\section{Conclusion}
\label{sec:conclusion}

In this paper, we introduce the Risk-Constrained Bilevel Stackelberg Framework (RCBSF), a novel framework designed to resolve the critical trade-off between risk mitigation and semantic preservation in high-stakes contract automation. RCBSF formalizes revision as a hierarchical game, decoupling strategic auditing from execution. The framework employs a Global Prescriptive Agent (GPA) to guide a Follower system, which comprises the Constrained Revision Agent (CRA) and Local Verification Agent (LVA). The GPA enforces strict constraints to ensure the Follower iteratively optimizes clauses within precise bounds. Evaluations on legal benchmarks confirm that RCBSF effectively balances strict risk compliance with linguistic fidelity. This establishes a new paradigm for reliable, constraint-aware legal text generation.


\section*{Limitations}
\label{sec:limitations}

One limitation of our research lies in the inherent properties of the dataset. The contract clauses in our unified benchmark are predominantly derived from publicly available authoritative datasets (e.g., MAUD, ContractNLI, CUAD). These clauses are structurally consistent and rarely involve highly complex cross-clause conflicts or niche scenario disputes. We have enriched data diversity through multi-source integration. However, more realistic and intricate scenarios remain insufficiently explored. Specifically, implicit unfairness arising from the interaction of multiple vague or inconsistent clauses, as well as interdependent risks that span multiple contract sections, still lack in-depth investigation.

Another limitation lies in the jurisdictional and linguistic bias of our validation. The datasets used for validation are predominantly sourced from United States and Common Law jurisdictions and are entirely in English. Consequently, the efficacy of our model in Civil Law systems (e.g., Germany, France) or non-English languages remains unverified. Legal concepts such as Force Majeure or Indemnification vary significantly across legal systems, and our current prompt engineering may implicitly encode US-centric legal norms.


\section*{Ethical Considerations}
\label{sec:ethics}

Deploying Large Language Models in high-stakes legal domains introduces inherent risks regarding hallucination and potential bias. While our RCBSF framework aims to mitigate errors through evidence grounding and token efficiency, it is designed strictly as an \textit{assistive tool} rather than a substitute for qualified legal counsel. Users must exercise appropriate care and verify all automated revisions to avoid automation bias. Regarding data privacy, all experiments utilized publicly available datasets (e.g., CUAD, PrivacyQA) with permissible licenses, and we applied strict de-identification pipelines to ensure no Personally Identifiable Information (PII) was exposed.

\section*{Acknowledgements}
This work was supported by the National Natural Science Foundation of China (No. 52578347).



\bibliography{main}



\newpage
\appendix

\section*{Appendices}

Within this supplementary material, we elaborate on the following aspects:

\begin{itemize}[leftmargin=*, label=$\bullet$]
    \item Appendix~\ref{sec:appendix_data}: Dataset Construction and Preprocessing
    \item Appendix~\ref{sec:appendix_proofs}: Mathematical Proofs and Derivations
    \item Appendix~\ref{sec:appendix_implementation}: Implementation Details
    \item Appendix~\ref{sec:appendix_evaluation}: Evaluation Metrics and Protocols
    \item Appendix~\ref{sec:appendix_results}: Detailed Experimental Results
    \item Appendix~\ref{sec:appendix_cases}: Qualitative Case Studies
    \item Appendix~\ref{sec:appendix_prompt}: Prompt Engineering
\end{itemize}

\section{Dataset Construction and Preprocessing}
\label{sec:appendix_data}

To evaluate the robustness of the \textsc{RCBSF} framework, we constructed a unified legal contract revision benchmark. This dataset was derived from four authoritative legal NLP corpora: \textsc{CUAD} \citep{atticus2021cuad}, \textsc{MAUD} \citep{roit2023maud1}, \textsc{ContractNLI} \citep{koreeda2021contractnli}, and \textsc{PrivacyQA} \citep{ravichander2019privacyqa}. 

However, raw legal documents often contain inconsistent formatting, excessive Personally Identifiable Information (PII), and unstructured clauses that hinder standardized evaluation. To address this, we implemented a rigorous \textit{Template Standardization and Risk Enrichment Pipeline}. The pipeline consists of three distinct stages designed to transform raw text into high-quality, privacy-preserved, and risk-annotated samples.

\subsection{Data Distribution}
To ensure the model's robust generalization capability across heterogeneous legal domains, we constructed a comprehensive dataset comprising 711 legal contracts spanning 41 distinct categories. As detailed in Table~\ref{tab:dataset_stats}, this unified benchmark aggregates high-quality clauses from four diverse sources: CUAD, MAUD, ContractNLI, and PrivacyQA. 

The dataset exhibits a long-tail distribution as illustrated in Figure~\ref{fig:contract_dist}. Among the 41 categories, \textit{Merger} agreements constitute the largest proportion at 18.4\%, followed by \textit{Disclosure} at 8.6\%. Other significant categories include \textit{Manufacturing} (4.5\%) and \textit{Marketing} (4.3\%), while specialized domains such as \textit{Confidentiality} (3.0\%) and \textit{Affiliate} agreements (1.3\%) ensure coverage of niche legal contexts. This diverse composition is critical for evaluating the model's adaptability to varying terminologies and clause structures.

\begin{table}[htb]
\centering
\resizebox{\linewidth}{!}{%
\begin{tabular}{lcc}
\hline
\textbf{Dataset} & \textbf{Num of Samples} & \textbf{Num of Categories} \\
\hline
PrivacyQA & 7 & 2 \\
ContractNLI & 92 & 7 \\
MAUD & 150 & 7 \\
CUAD & 462 & 25 \\
\hline
\textbf{Total} & \textbf{711} & \textbf{41} \\
\hline
\end{tabular}%
}
\caption{Statistics of the processed Unified Legal Benchmark. The dataset aggregates 711 samples from diverse legal sources, covering 41 distinct categories to ensure broad domain coverage.}
\label{tab:dataset_stats}
\end{table}

\begin{figure}[h]
    \centering
    \includegraphics[width=1\linewidth]{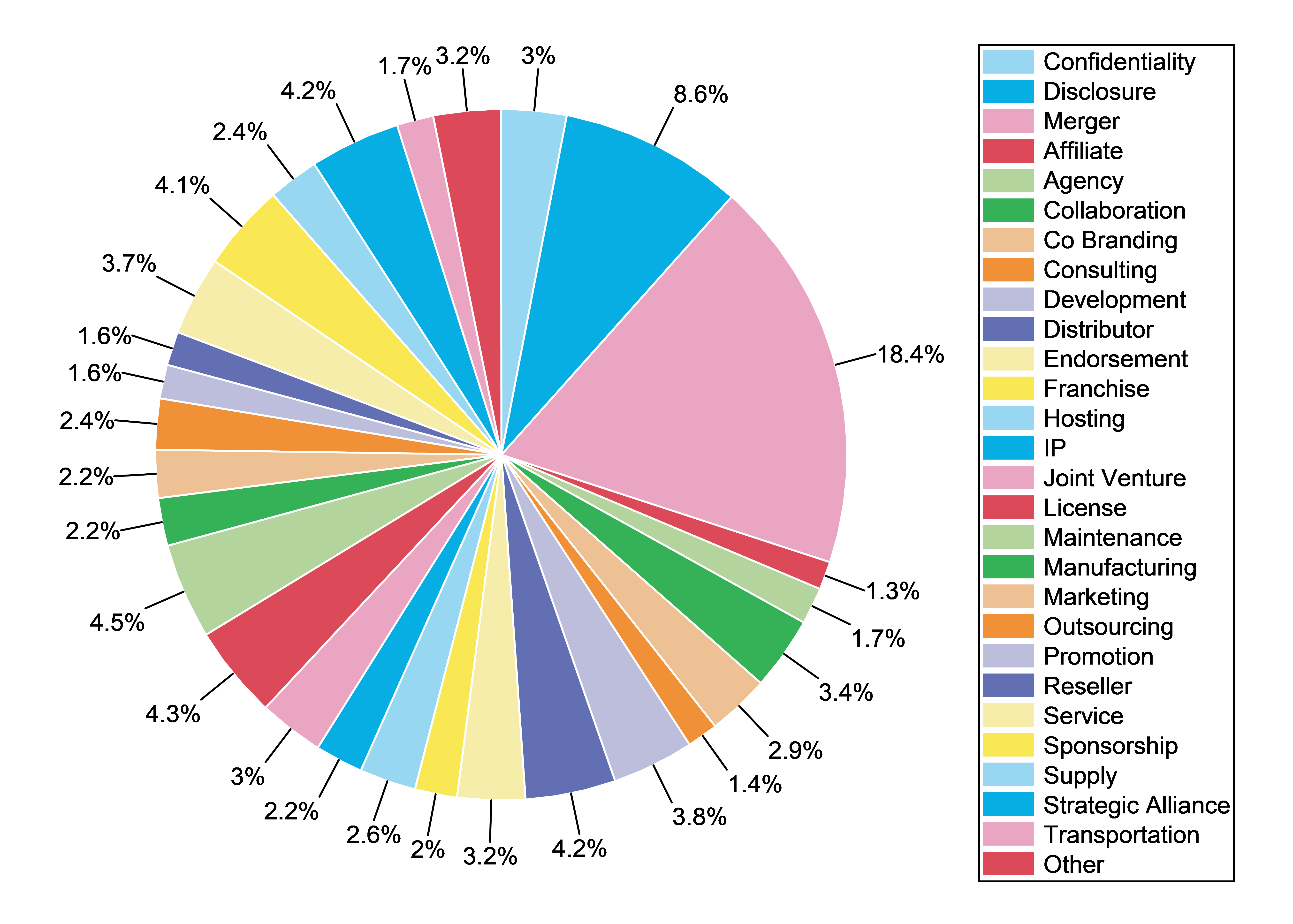}
    \caption{Distribution of Contract Categories. The dataset features a long-tail distribution across 41 legal domains, with significant representation in major categories like Merger and Disclosure.}
    \label{fig:contract_dist}
\end{figure}

\subsection{Data Processing Pipeline}

This section details the multi-stage pipeline designed to transform raw legal documents into a high-quality, annotated dataset. The process moves from initial classification to content standardization, risk enrichment, and finally, granular legal provision recommendation.

\subsubsection{Automated Category Classification (Stage 1)}

The initial phase of the pipeline focuses on ingesting heterogeneous file headers and normalizing them into a canonical taxonomy. As raw legal documents often contain inconsistent naming conventions and formatting artifacts, a rule-based approach is employed to ensure data integrity before meaningful processing begins.

As illustrated in Figure~\ref{fig:stage0_pipeline}, the system executes a three-step logic. First, the Ingestion phase extracts raw text headers from source files (PDF/TXT). This is followed by Noise Reduction, where a regex-based cleaning function ($f_{clean}$) is applied to remove file extensions, numbering artifacts, and redundant whitespace. Finally, the Canonicalization step employs a hierarchical keyword mapping algorithm to assign the cleaned header to a target category ($C_{target}$), utilizing logic such as mapping NDA, Confidentiality, or Non-Disclosure to a unified \textit{NDA} class.

\begin{figure}[htb]
    \begin{tcolorbox}[
        colback=white, 
        colframe=gray!50!black, 
        title=\textbf{Stage 1: Automated Category Classification Pipeline},
        fonttitle=\bfseries,
        arc=1mm
    ]
    \small
    \textbf{1. Input Ingestion} \\
    Raw PDF/TXT text headers are ingested ($H_{raw}$).
        
    \textbf{2. Noise Reduction (Regex)} \\
    Apply cleaning function $f_{clean}(H_{raw})$: 
    \begin{enumerate}[label=\alph*), leftmargin=*, topsep=0pt, itemsep=0pt, parsep=0pt]    
        \item Remove file extensions (`.pdf`, `.docx`).
        \item Strip numbering artifacts .
        \item Normalize whitespace to single space.
    \end{enumerate}
    
    \textbf{3. Keyword Mapping (Canonicalization)} \\
    Map headers to $C_{target} \in \mathcal{C}$ using hierarchical keyword matching: 
    \begin{enumerate}[label=\alph*), leftmargin=*, topsep=0pt, itemsep=0pt, parsep=0pt]
        \item \texttt{IF} "license" AND "software" $\rightarrow$ \textit{Software License}
        \item \texttt{IF} "consulting" OR "service" $\rightarrow$ \textit{Service Agreement}
        \item \texttt{IF} "confidential" OR "nds" $\rightarrow$ \textit{NDA}
    \end{enumerate}

    \textbf{4. Output} \\
    Structured pair: \texttt{(Category\_ID, Clean\_Header)}.
    \end{tcolorbox}
    \caption{The automated classification logic used in Stage 0 to standardize noisy raw file headers into canonical legal categories.}
    \label{fig:stage0_pipeline}
\end{figure}

\subsubsection{Template Standardization (Stage 2)}

Following classification, the raw contract text undergoes a standardization process to remove identifiers and unify structure. This is achieved using a Large Language Model (ChatGPT-5) configured to act as a generic summarizer. The objective is to produce a clean slate template that retains the legal intent of the original document while neutralizing specific party details.

As presented in Figure~\ref{fig:stage1_prompt}, the prompt engineering constraints specifically instruct the model to perform three key functions. First, it must normalize the structure by reorganizing the content into seven standard sections (e.g., Definitions, Indemnification, Termination). Second, it is required to anonymize PII by replacing sensitive entities and values with placeholders such as \texttt{[Party A]} and \texttt{[Amount]}. Lastly, the model must output plain text, ensuring the result is machine-readable without markdown formatting.

\begin{figure}[htb]
    \begin{tcolorbox}[
        colback=white, 
        colframe=gray!50!black, 
        title=\textbf{Stage 2: Summarizer Prompt},
        fonttitle=\bfseries,
        arc=1mm
    ]
    \small
    \textbf{System Role:} 
    You are a Senior Legal Counsel with 20 years of experience in contract law. Your goal is to draft precise, neutral, and enforceable contract templates.

    \textbf{Task Instructions:}
    Transform the provided source text into a standardized template following these strict constraints:
    
    \textbf{1. Structural Normalization:}
    Organize the output into these exact sections:
    \begin{itemize}[leftmargin=*, nosep]
        \item Definitions
        \item Scope of Services
        \item Fees and Payment
        \item IP Ownership
        \item Confidentiality \& Data Protection
        \item Indemnification \& Liability
        \item Term \& Termination
    \end{itemize}

    \textbf{2. PII Anonymization (Strict):}
    \begin{itemize}[leftmargin=*, nosep]
        \item Replace entity names with \texttt{[Party A]} / \texttt{[Party B]}.
        \item Replace specific dates with \texttt{[Effective Date]}.
        \item Replace monetary values with \texttt{[Amount]}.
    \end{itemize}

    \textbf{3. Formatting:}
    Output \textbf{PLAIN TEXT ONLY}. Do not use Markdown bolding or code blocks. Keep length $\le 1500$ words.
    \end{tcolorbox}
    \caption{The instruction set for LLM used to generate standardized, anonymized legal templates from raw source text.}
    \label{fig:stage1_prompt}
\end{figure}

\subsubsection{Risk Enrichment (Stage 3)}

To facilitate the evaluation of risk detection models, the standardized templates are processed by a new Large Language Model (ChatGPT-5), which functions as an adversarial auditor. This stage generates synthetic standard truth data by simulating a high-scrutiny legal review.

As detailed in Figure~\ref{fig:stage2_prompt}, the prompt directs the model to identify specific vulnerabilities within the contract (e.g., uncapped liability) and propose corresponding mitigations. Crucially, the prompt includes negative constraints to filter out trivial errors (such as typos or missing signatures), ensuring the generated dataset focuses on substantive legal risks. The output is structured as a JSON object containing risk categories, issue descriptions, and mitigation strategies.

\begin{figure}[htb]
    \begin{tcolorbox}[
        colback=white, 
        colframe=gray!50!black, 
        title=\textbf{Stage 3: Auditor Prompt},
        fonttitle=\bfseries,
        arc=1mm
    ]
    \small
    \textbf{System Role:}
    You are an Adversarial Risk Auditor. Your job is to stress-test contracts for operational and legal vulnerabilities.

    \textbf{Input:} Standardized Contract Template ($T_{std}$).

    \textbf{Task Instructions:}
    Identify 8--12 specific risks. For each risk, generate a "Risk-Mitigation" pair.

    \textbf{Negative Constraints (FILTER):}
    \textit{DO NOT} generate risks related to:
    \begin{itemize}[leftmargin=*, nosep]
        \item Missing signatures or dates.
        \item Typos or formatting errors.
        \item Generic "Governing Law" preferences.
    \end{itemize}

    \textbf{Output Format (JSON):}
    \begin{verbatim}
    {
      "risks": [
        {
          "category": "IP Indemnification",
          "issue": "Clause fails to cap liability 
                    for third-party IP claims.",
          "mitigation": "1. Insert liability cap.
                         2. Carve out exceptions 
                         for gross negligence."
        }
      ]
    }
    \end{verbatim}
    \end{tcolorbox}
    \caption{The adversarial prompt for LLM, designed to generate the standard truth risk dataset used for calculating the Risk Resolution Rate (RRR).}
    \label{fig:stage2_prompt}
\end{figure}

\section{Mathematical Proofs and Derivations}
\label{sec:appendix_proofs}

\begin{figure}[h]
    \centering
    \includegraphics[width=1\linewidth]{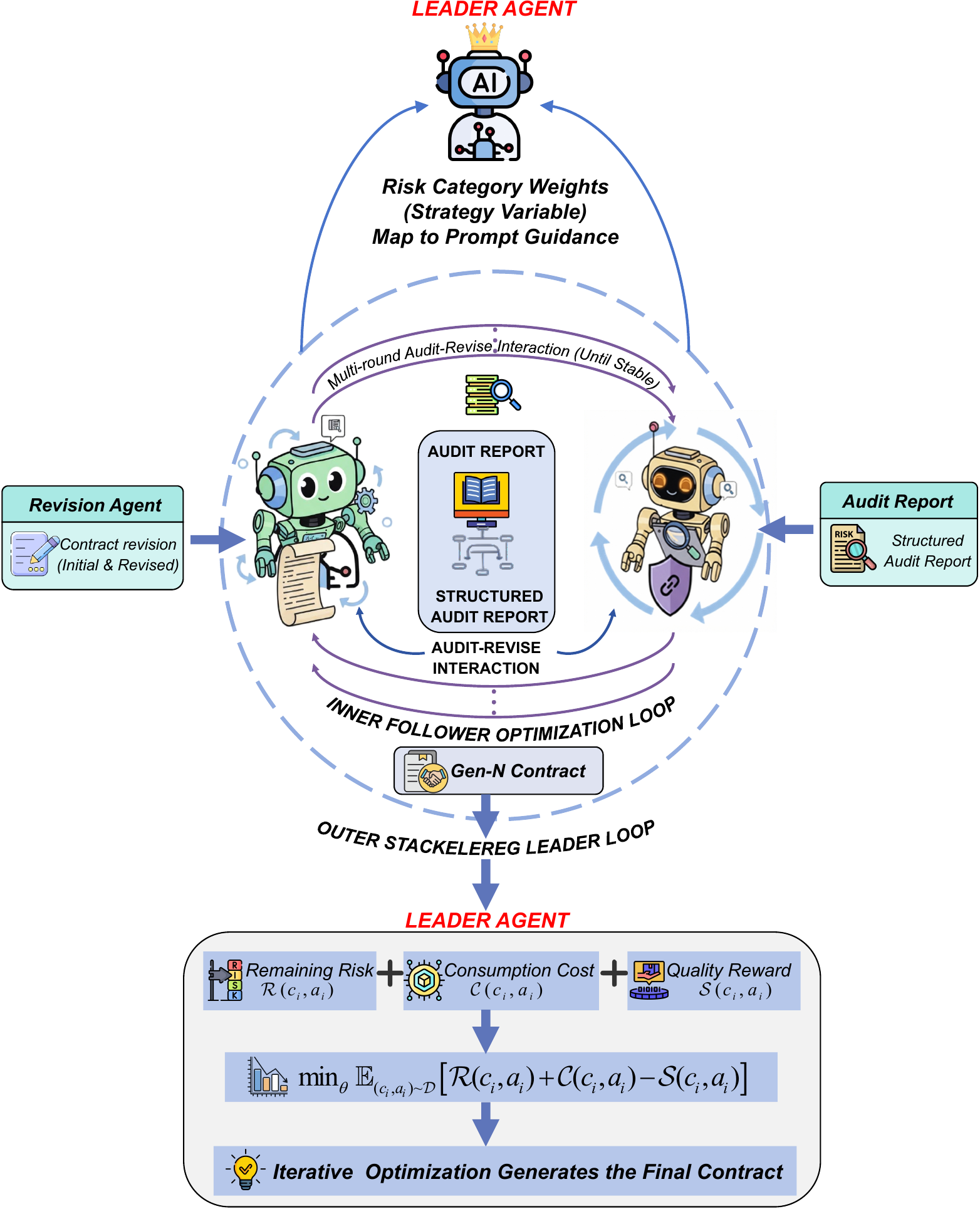}
    \caption{\label{fig:rcbsf_framework} The architectural overview of the Risk-Constrained Bilevel Stackelberg Framework (RCBSF). The system is organized into two strategic levels: 
    (1) The Outer Stackelberg Leader Loop , where the Leader Agent defines the strategic environment by mapping risk category weights to prompt guidance and optimizing the global utility functional. 
    (2) The Inner Follower Optimization Loop, where the Revision Agent and Audit Agent engage in a multi-round adversarial game. The cycle continues until the generation stabilizes, ensuring the final contract $x^*$ satisfies the Leader's constraints via the minimization objective.}
    \label{fig:boyi}
\end{figure}

In this appendix, we provide the rigorous measure-theoretic foundations and proofs for the theorems presented in Section~\ref{sec:methodology}. We analyze the properties of the RCBSF framework on a continuous semantic manifold using tools from functional analysis and topology. 

Specifically, we elucidate the mechanics of the hierarchical interaction illustrated in Figure~\ref{fig:boyi}. The framework operates through a bilevel optimization strategy. The Outer Stackelberg Leader Loop governs the strategic commitment, where the Leader Agent maps discrete risk category weights to continuous prompt guidance vectors, effectively constraining the search space. Conversely, the Inner Follower Optimization Loop executes the tactical refinement through a multi-round audit-revise interaction. Within this loop, the Revision Agent and Audit Agent iteratively minimize the local divergence until a stable equilibrium is reached.

\subsection{Geometric Construction of the Risk Manifold}
\label{app:manifold}

The contract text $\mathbf{x}$ is embedded in a high-dimensional Riemannian manifold $\mathcal{M} \subseteq \mathbb{R}^d$. The \textit{Global Prescriptive Agent} (GPA) projects $\mathbf{x}$ onto a structured risk subspace.

\begin{definition}[5-Dimensional Risk Tensor]
Let the risk space $\mathcal{V}$ be defined as the Cartesian product of five distinct topological feature spaces:
\begin{equation}
    \mathcal{V} \triangleq \mathcal{C} \times \mathcal{L} \times \mathcal{E} \times \mathcal{I} \times \mathcal{S}
\end{equation}
where $\mathcal{C}$ is the categorical distribution space, $\mathcal{L}$ is the localization metric space, $\mathcal{E}$ is the evidential text subspace, $\mathcal{I}$ is the semantic issue embedding, and $\mathcal{S}$ is the suggestion vector field.
\end{definition}

The instruction vector $\mathbf{h}$ is formalized as a Dirac mixture distribution over $\mathcal{V}$:
\begin{equation}
    \mathbf{h}(\mathbf{v}) = \sum_{k=1}^K \alpha_k \delta(\mathbf{v} - \mathbf{v}_k), \quad \mathbf{v}_k \in \mathcal{V}
\end{equation}
where $\delta(\cdot)$ is the Dirac delta function and $\alpha_k$ represents the attention weight derived from the Q-score weights $\mathbf{w}_Q$.

\subsection{Existence of Stackelberg Equilibrium}

We first establish the existence of a solution to the bilevel optimization problem.

\begin{lemma}[Compactness of Strategy Space]
Assuming the semantic embedding space is bounded via layer normalization, the set of admissible instructions $\mathcal{H}_B = \{ \mathbf{h} \in \mathcal{H} \mid \|\mathbf{h}\|_0 \le B \}$ is compact in the discrete topology induced by the token vocabulary.
\end{lemma}

\begin{proposition}[Existence]
Let the Leader's utility $J_L: \mathcal{M} \times \mathcal{H} \to \mathbb{R}$ be continuous, and the Follower's reaction set $\Psi(\mathbf{h}) = \arg\max_{\mathbf{x}'} J_F(\mathbf{x}', \mathbf{x}, \mathbf{h})$ be upper hemicontinuous. Then, a Stackelberg Equilibrium $(\mathbf{h}^*, \mathbf{x}^*)$ exists.
\end{proposition}

\begin{proof}
Since $\mathcal{H}_B$ is compact (finite set of token sequences with length $\le B$), and $J_L$ is a continuous mapping of risk reduction, by the \textbf{Weierstrass Extreme Value Theorem}, $J_L$ attains its maximum on the graph of $\Psi$. Thus, there exists $\mathbf{h}^* \in \mathcal{H}_B$ such that:
\begin{equation}
    \mathbf{h}^* = \underset{\mathbf{h} \in \mathcal{H}_B}{\arg\max} \left( \min_{\mathbf{x} \in \Psi(\mathbf{h})} J_L(\mathbf{x}, \mathbf{h}) \right)
\end{equation}
This confirms that the  function is mathematically guaranteed to find an optimal prompt configuration.

\end{proof}

\subsection{Proof of Strict Superiority (Theorem~\ref{thm:superiority})}

We prove that the Stackelberg strategy strictly dominates the unguided (Nash) strategy.

\begin{proof}
Let $\mathcal{V}_{SE}$ and $\mathcal{V}_{NE}$ denote the optimal values of the Leader's objective under Stackelberg and Nash equilibria, respectively.
\begin{enumerate}
    \item \textbf{Inclusion Property:} The unguided generation corresponds to the null instruction strategy $\mathbf{h}_{\emptyset} = \emptyset$. Since the budget $B > 0$, we have $\emptyset \in \mathcal{H}_B$.
    \item \textbf{Global Optimality:} The Stackelberg leader maximizes over the entire set $\mathcal{H}_B$. Therefore:
    \begin{equation}
    \begin{split}
        \mathcal{V}_{SE} &= \sup_{\mathbf{h} \in \mathcal{H}_B} J_L(\Psi(\mathbf{h}), \mathbf{h}) \\
        &\ge J_L(\Psi(\emptyset), \emptyset) = \mathcal{V}_{NE}
    \end{split}
    \end{equation}
    \item \textbf{Strict Inequality Condition:}
    Let $\mathcal{R}(\mathbf{x}) = \sum w_r \delta(r, \mathbf{x})$ be the risk potential. In the unguided case, the LLM generates $\mathbf{x}_{NE}$ by sampling from $P(\mathbf{x}|\emptyset)$. Due to misalignment, there exists a risk $r_k$ such that $\delta(r_k, \mathbf{x}_{NE}) > \epsilon$.
    
    The GPA constructs a specific hint $\mathbf{h}_k$ containing the \textit{Evidence} $e_k$ and \textit{Suggestion} $s_k$. This introduces a forcing term in the Follower's objective:
    \begin{equation}
        J_F(\mathbf{x}, \mathbf{h}_k) \approx J_F(\mathbf{x}, \emptyset) - \gamma \|\phi(\mathbf{x}) - \phi(s_k)\|^2
    \end{equation}
    Provided $\gamma$ is sufficiently large, the global minimum of $J_F(\cdot, \mathbf{h}_k)$ shifts to a region where $\delta(r_k, \mathbf{x}) < \epsilon$. Since the cost of hint $C(\mathbf{h}_k)$ is negligible compared to the risk penalty $w_r$, we have $J_L(\mathbf{h}_k) > J_L(\emptyset)$. Hence $\mathcal{V}_{SE} > \mathcal{V}_{NE}$.
\end{enumerate}
\end{proof}

\subsection{Convergence of Inner Iterative Refinement (Theorem~\ref{thm:convergence})}

We analyze the convergence of the inner loop using the Banach Fixed-Point Theorem.

\begin{definition}[Fusion Operator]
Let $\mathcal{T}: \mathcal{M} \to \mathcal{M}$ be the transition operator defined by one iteration of the Inner Drafter guided by the Local Verifier.
\begin{multline}
    \mathbf{x}^{(t+1)} = \mathcal{T}(\mathbf{x}^{(t)}) \\
    = \text{CRA}\left( \mathbf{x}^{(t)}, \mathcal{F}_{\text{fusion}}(\mathbf{Q}_{\text{outer}}, \mathbf{Q}_{\text{inner}}^{(t)}) \right)
\end{multline}
\end{definition}

\begin{proof}
Consider two semantic states $\mathbf{x}, \mathbf{y} \in \mathcal{M}$. The Drafter updates them based on the gradient of the fused risk score.
\begin{align}
    & d\bigl(\mathcal{T}(\mathbf{x}), \mathcal{T}(\mathbf{y})\bigr) \nonumber \\
    &\quad = \bigl\| \bigl(\mathbf{x} - \eta \nabla \mathcal{L}(\mathbf{x})\bigr) - \bigl(\mathbf{y} - \eta \nabla \mathcal{L}(\mathbf{y})\bigr) \bigr\| \nonumber \\
    &\quad \le \|\mathbf{x} - \mathbf{y}\| + \eta \bigl\| \nabla \mathcal{L}(\mathbf{x}) - \nabla \mathcal{L}(\mathbf{y}) \bigr\| \label{eq:7}
\end{align}
Assuming the risk landscape $\mathcal{L}$ is $L$-smooth (Lipschitz continuous gradients) and strongly convex in the neighborhood of the optimum, with learning rate $\eta < 2/L$, the operator satisfies:
\begin{equation}
    d(\mathcal{T}(\mathbf{x}), \mathcal{T}(\mathbf{y})) \le \kappa \cdot d(\mathbf{x}, \mathbf{y})
\end{equation}
where $\kappa \in [0, 1)$. By the Banach Fixed Point Theorem, the sequence $\{\mathbf{x}^{(t)}\}$ generated by the inner loop converges uniquely to a fixed point $\mathbf{x}^*$.
\end{proof}

\section{Implementation Details}
\label{sec:appendix_implementation}

In this section, we provide a comprehensive breakdown of the experimental environment, hyperparameter configuration, and the sensitivity analysis that guided our architectural choices.

\subsection{Experimental Setup}
All experiments were conducted on a high-performance computing cluster equipped with $8\times$ NVIDIA A100 (80GB) GPUs. The core framework was implemented in Python 3.11 using the PyTorch ecosystem. For the backbone Large Language Model (LLM), we utilized Qwen2.5-7B-Chat, served via a local inference wrapper to ensure reproducibility and data privacy. We employed a consistent random seed ($seed=42$) across all generation tasks to mitigate non-deterministic variance.

\subsection{Hyperparameter Sensitivity and Selection}
\label{hyperparameter}
The performance of the \textsc{RCBSF} framework relies on three critical hyperparameters: the Stackelberg game iteration depth ($K$), the LVA's softmax temperature ($\tau$), and the risk weighting vector ($\mathbf{w}$). We determined the optimal values for these parameters through extensive sensitivity analysis on the validation set.

\paragraph{Iteration Rounds ($K$).} 
We treat the contract revision process as a finite-horizon Stackelberg game. Determining the optimal number of interaction rounds is a trade-off between risk resolution effectiveness and computational efficiency. 
As illustrated in Figure~\ref{fig:efficiency_tradeoff}, the Risk Resolution Rate (RRR) follows a law of diminishing returns.
\begin{itemize}
    \item \textbf{Rapid Gain Phase ($K=1 \to 2$):} The RRR surges from $74.24\%$ to $82.12\%$, indicating that the initial feedback loop is crucial for correcting obvious semantic errors.
    \item \textbf{Convergence Phase ($K=3$):} The performance peaks at $84.21\%$. Beyond this point, the marginal gain is negligible (e.g., $K=8$ yields only $84.24\%$), while the computational cost (Normalized Token Cost) continues to grow linearly from $2.8\times$ to $7.3\times$.
\end{itemize}
Consequently, we set the stopping condition at $K=3$ to maximize the utility-cost ratio.

\begin{figure}[h]
    \centering
    \includegraphics[width=1\linewidth]{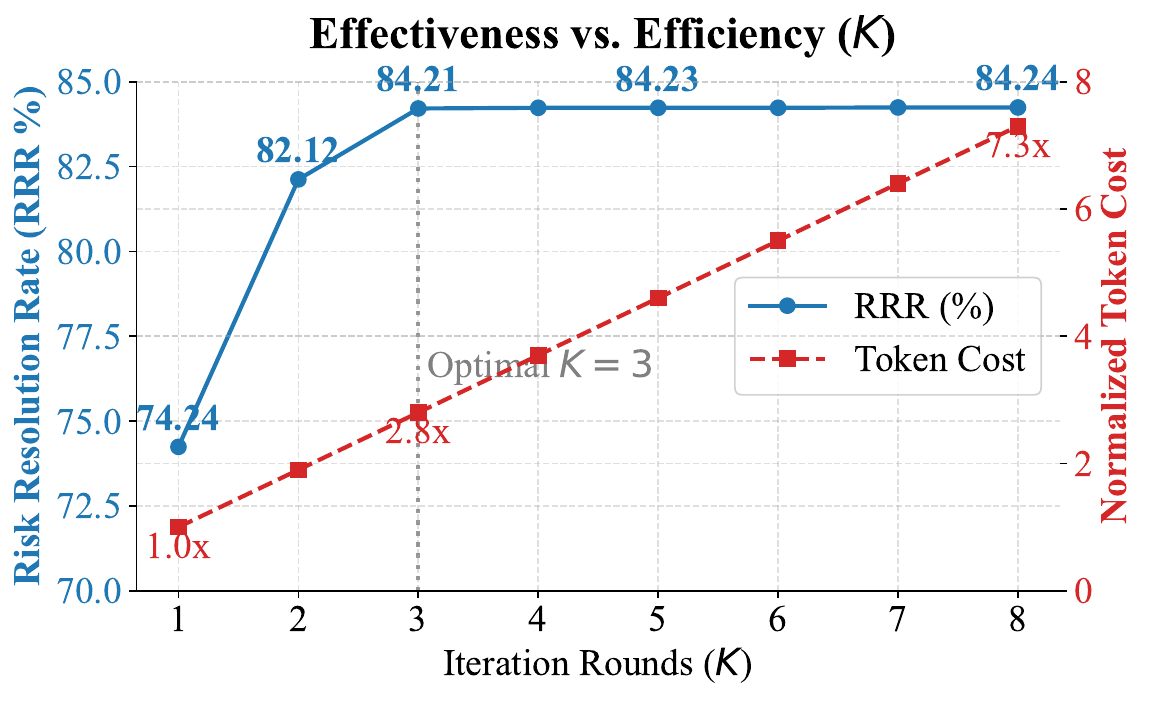} 
    \caption{Effectiveness vs. Efficiency ($K$). The left axis (blue) shows the Risk Resolution Rate, while the right axis (red) shows the normalized Token Cost. $K=3$ represents the optimal trade-off point.}
    \label{fig:efficiency_tradeoff}
\end{figure}

\paragraph{Softmax Temperature ($\tau$).}
The GPA (Leader Agent) utilizes a softmax function to convert risk severity scores ($Q$-scores) into a probabilistic strategy vector. The temperature parameter $\tau$ controls the entropy of this distribution.
As shown in Figure~\ref{fig:temp_impact}, we observe a distinct bell-shaped performance curve:
\begin{itemize}
    \item \textbf{Conservative Regime ($\tau < 0.5$):} Low temperatures lead to an overly greedy strategy that focuses only on the most obvious risks, missing subtle long-tail issues (RRR $\approx 76.54\%$ at $\tau=0.1$).
    \item \textbf{High-Entropy Regime ($\tau > 1.5$):} High temperatures introduce excessive noise into the instruction vector, causing the CRA to lose focus on critical constraints (RRR drops to $72.32\%$ at $\tau=2.0$).
    \item \textbf{Optimal Balance ($\tau=1.0$):} We find that $\tau=1.0$ achieves the global maximum (RRR $84.21\%$), ensuring the LVA's instructions are both decisive and sufficiently diverse to cover multi-dimensional risks.
\end{itemize}

\begin{figure}[h]
    \centering
    \includegraphics[width=1\linewidth]{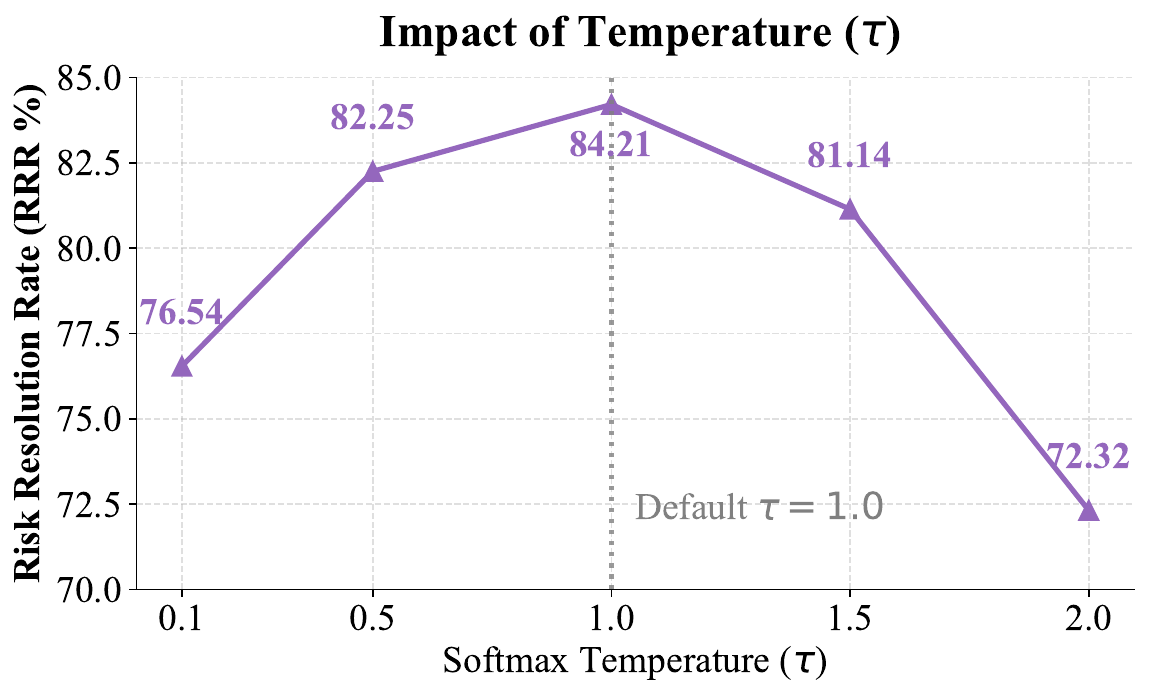} 
    \caption{Impact of Temperature ($\tau$). The curve demonstrates that a balanced temperature ($\tau=1.0$) significantly outperforms both conservative ($\tau=0.5$) and high-entropy ($\tau=2.0$) settings.}
    \label{fig:temp_impact}
\end{figure}

\paragraph{Risk Weighting Vector ($\mathbf{w}$).}
To map the discrete Q-scores ($Q1\dots Q4$) extracted by the GPA into a scalar severity metric, we employ a fixed weight vector $\mathbf{w} = [0.4, 0.2, 0.2, 0.2]$. This configuration assigns the highest importance to $Q1$ (Legal Liability Severity), while equally balancing $Q2$ (Modification Difficulty), $Q3$ (Hidden Risk Probability), and $Q4$ (Language Ambiguity). This weighting ensures that critical liability loopholes prioritize the CRA's limited token budget.

\subsection{Budget Constraints and Optimization}
The \textsc{RCBSF} framework imposes strict token budgets to simulate real-world API cost constraints. Based on the average length of commercial contracts in our dataset, we enforced the following constraints during the optimization process:
\begin{itemize}
    \item \textbf{Contract Budget ($\beta_{contract}$):} Set to $3,000$ tokens. This forces the Revision to be concise and prevents the verbose generation failure mode common in standard LLMs.
    \item \textbf{Audit Budget ($\beta_{audit}$):} Set to $1,500$ tokens. This constrains the LVA's feedback to be high-density and actionable, prioritizing the Top-$k$ most severe risks rather than listing trivial stylistic suggestions.
\end{itemize}
These constraints are implemented via a penalty term in the Leader's objective function (Equation~\ref{sec:methodology} in the main paper), ensuring that any generation exceeding these limits incurs a high utility cost.

\section{Evaluation Metrics and Protocols}
\label{sec:appendix_evaluation}

To rigorously assess the performance of the \textsc{RCBSF} framework, we employed a multi-dimensional evaluation protocol encompassing automated verification, LLM-based judicial scoring, and expert human review. This section details the mathematical formulations of our metrics and the specific prompt engineering used for the Judge Agent.

\subsection{Automated Risk Verification Protocol}
\label{subsec:eval_rrr}

The primary safety metric, Risk Resolution Rate (RRR), measures the system's ability to effectively mitigate specific, pre-identified legal risks. We formalize this verification process as a binary classification task performed by an independent Judge Agent (GPT-5).

Let $\mathcal{C}$ be the set of contract clauses and $\mathcal{R} = \{r_1, r_2, \dots, r_n\}$ be the set of Golden High Risk items associated with contract $x$. Each risk $r_i$ is defined as a tuple $r_i = \langle c_{cat}, c_{loc}, c_{sug} \rangle$, representing the risk category, location text, and golden revision suggestion, respectively.

The Judge Agent function $J_{\theta}(x', r_i)$ evaluates the revised contract $x'$ against risk $r_i$ to determine a resolution status $y_i \in \{0, 1\}$ and a confidence score $\sigma_i \in [0, 1]$. The verification process is governed by the following prompt logic:

\begin{equation}
y_i = \mathbb{I}\left( \text{Judge}(x' | c_{cat}, c_{basis}) = \text{"resolved"} \right)
\end{equation}

where $\mathbb{I}(\cdot)$ is the indicator function. The Judge is explicitly instructed to look for verifiable elements such as numerical values, time limits, or specific remedies.

The aggregate Risk Resolution Rate (RRR) for a given method $M$ across a dataset $\mathcal{D}$ is calculated as:

\begin{equation}
\text{RRR}(M) = \frac{1}{|\mathcal{D}|} \sum_{j=1}^{|\mathcal{D}|} \left( \frac{\sum_{i=1}^{|R_j|} y_{j,i} \cdot \sigma_{j,i}}{\sum_{i=1}^{|R_j|} 1} \right) \times 100\%
\end{equation}

To ensure robustness, we utilized the prompt structure shown in Figure~\ref{fig:prompt_judge_rrr}, which forces the model to output a rationale before the binary verdict.

\begin{figure}[htp]
\begin{tcolorbox}[colback=white, colframe=gray!50!black, title=\textbf{Prompt for Risk Resolution Judge}]
\small
\textbf{System:} You are a rigorous and neutral contract risk control reviewer.
\textbf{Input:}
Full contract text: \{contract\}
Risk Basis:
- Category: \{cat\}
- Basis/Details: \{basis\}

\textbf{Judgment Criteria:} If the contract has substantively mitigated or eliminated the risk through clause setting/definition/boundary/remedy with \textbf{verifiable elements} (numerical values, time limits, caps), mark it as resolved.

\textbf{Output JSON:}
\{"seed\_id": \{sid\}, "resolved": true/false, "confidence": 0-1, "rationale": "..."\}
\end{tcolorbox}
\caption{The strict verification prompt used to calculate RRR. It requires evidence of substantive changes rather than mere stylistic edits.}
\label{fig:prompt_judge_rrr}
\end{figure}

\subsection{Contract Quality (CQ) Metric}
\label{subsec:eval_cq}

Beyond risk mitigation, the linguistic and legal quality of the generated text is paramount. We define Contract Quality (CQ) as a weighted composite score derived from four orthogonal dimensions: Clarity ($\mathcal{S}_{cla}$), Rigor ($\mathcal{S}_{rig}$), Balance ($\mathcal{S}_{bal}$), and Professionalism ($\mathcal{S}_{pro}$).

Each dimension is scored on a discrete integer scale $s \in [0, 100]$ by the Judge Agent, following a strict rubric (see Table~\ref{tab:cq_rubric}). The scoring function $f_{score}(\cdot)$ maps the latent quality of contract $x'$ to a vector $\mathbf{s} \in \mathbb{R}^4$.

The final CQ Score is computed as a weighted sum, where weights $\mathbf{w}_{cq}$ are empirically determined to prioritize legal enforceability (Rigor) and fairness (Balance):

\begin{equation}
\text{CQ}(x') = \mathbf{w}_{cq}^\top \cdot \begin{bmatrix} \mathcal{S}_{cla}(x') \\ \mathcal{S}_{rig}(x') \\ \mathcal{S}_{bal}(x') \\ \mathcal{S}_{pro}(x') \end{bmatrix}
\end{equation}

In our experiments, we set $\mathbf{w}_{cq} = [0.25, 0.30, 0.25, 0.20]$ to reflect the relative importance of logical rigor in automated revision.

\begin{table*}[htb]
\centering
\small
\renewcommand{\arraystretch}{1.2}
\begin{tabularx}{\textwidth}{@{}l >{\hsize=0.9\hsize}X >{\hsize=1.0\hsize}X >{\hsize=1.1\hsize}X @{}}
\toprule
\textbf{Dimension} & 
\textbf{High Quality (80--100)} \newline \scriptsize{\textit{90--100: Excellent; 80--89: Good}} & 
\textbf{Borderline (60--79)} \newline \scriptsize{\textit{70--79: Above Avg; 60--69: Barely Acceptable}} & 
\textbf{Deficient (0--59)} \newline \scriptsize{\textit{40--59: Problematic; 0--39: Severe}} \\ 
\midrule

\textbf{Clarity} & 
\textbf{Precise \& Unambiguous.} Key terms consistent; rights/obligations expressed in enforceable terms. 
\newline \textit{Defect:} Minor ambiguity not impairing execution. & 
\textbf{Workable but Vague.} Main obligations understandable; reliance on vague terms (e.g., ``reasonable'') without standards; occasional inconsistency. 
\newline \textit{Defect:} Requires supplementary interpretation. & 
\textbf{Chaotic \& Undefined.} Essential clauses lack operational criteria; contradictory statements; unclear boundaries of rights. 
\newline \textit{Defect:} Unusable; chaotic terminology. \\ 
\addlinespace

\textbf{Rigor} & 
\textbf{Logical \& Complete.} Clear hierarchy; precise cross-references; strong internal consistency. 
\newline \textit{Defect:} Minor numbering/formatting issues only. & 
\textbf{Loose Structure.} Usable but contains logical gaps; mixed organization; frequent incorrect references. 
\newline \textit{Defect:} Logical jumps exploitable in disputes. & 
\textbf{Disorganized.} ``Pile of clauses'' with overlapping provisions; missing references; no systematic organization. 
\newline \textit{Defect:} Cannot support dispute resolution. \\ 
\addlinespace

\textbf{Balance} & 
\textbf{Symmetrical.} Fair risk allocation; remedies apply equally; no significant unfair benefits. 
\newline \textit{Defect:} Minor non-core tilts. & 
\textbf{Surface Symmetry.} Appears balanced but hides practical tilts; force majeure or procedural clauses favor one party. 
\newline \textit{Defect:} Mostly fair core rights, but biased details. & 
\textbf{One-Sided.} Extreme asymmetry; unilateral termination/caps without reciprocity; effectively a single-party tool. 
\newline \textit{Defect:} Clear lack of reciprocal remedies. \\ 
\addlinespace

\textbf{Professionalism} & 
\textbf{Formal \& Standard.} Professional revision; standard formatting; consistent terminology; error-free. 
\newline \textit{Defect:} Occasional small typos. & 
\textbf{Mixed Quality.} Noticeable grammar/format issues; occasional colloquialisms or ambiguous everyday expressions. 
\newline \textit{Defect:} Reads like a draft or informal memo. & 
\textbf{Unprofessional.} Highly colloquial; messy formatting; riddled with spelling/grammar errors. 
\newline \textit{Defect:} Lacks legal terminology/structure. \\ 

\bottomrule
\end{tabularx}
\caption{Condensed scoring rubric for Contract Quality. The Judge Agent selects a band (High/Borderline/Deficient) based on the qualitative descriptors and fine-tunes the score within the specific 10-point ranges (e.g., 90--100 vs. 80--89).}
\label{tab:cq_rubric}
\end{table*}

The Win Rate (WR) for our model against baseline $B$ is calculated as:

\begin{equation}
\text{WR}(M, B) = \frac{N_{win} + 0.5 \times N_{tie}}{N_{total}}
\end{equation}

where $N_{win}$ is the count of pairwise comparisons where experts preferred $M$.

\section{Detailed Experimental Results}
\label{sec:appendix_results}

\begin{table*}[h]
\centering
\resizebox{\textwidth}{!}{%
\begin{tabular}{c|c|ccccc|ccccc|ccccc|ccccc|ccccc}
\toprule
\multirow{2}{*}{\textbf{Model}} & \multirow{2}{*}{\textbf{Method}} & \multicolumn{5}{c|}{\textbf{PrivacyQA}} & \multicolumn{5}{c|}{\textbf{ContractNLI}} & \multicolumn{5}{c|}{\textbf{MAUD}} & \multicolumn{5}{c|}{\textbf{CUAD}} & \multicolumn{5}{c}{\textbf{ALL (Avg)}} \\
\cmidrule{3-27}
 &  & \textbf{Cla} & \textbf{Rig} & \textbf{Bal} & \textbf{Pro} & \textbf{Avg} & \textbf{Cla} & \textbf{Rig} & \textbf{Bal} & \textbf{Pro} & \textbf{Avg} & \textbf{Cla} & \textbf{Rig} & \textbf{Bal} & \textbf{Pro} & \textbf{Avg} & \textbf{Cla} & \textbf{Rig} & \textbf{Bal} & \textbf{Pro} & \textbf{Avg} & \textbf{Cla} & \textbf{Rig} & \textbf{Bal} & \textbf{Pro} & \textbf{Avg} \\
\midrule
\multirow{5}{*}{Qwen2.5-7B-Base}
 & Standard  & 73.12 & 68.45 & 70.33 & 73.89 & 71.45 & 72.11 & 66.89 & 69.23 & 72.25 & 70.12 & 71.34 & 67.56 & 69.12 & 71.51 & 69.88 & 71.67 & 67.45 & 69.34 & 72.45 & 70.23 & 72.06 & 67.59 & 69.51 & 72.53 & 70.42 \\
 & CoT  & 75.34 & 71.12 & 72.45 & 76.65 & 73.89 & 74.23 & 69.56 & 71.34 & 74.67 & 72.45 & 73.56 & 69.89 & 71.23 & 73.79 & 72.12 & 74.12 & 69.78 & 71.56 & 74.77 & 72.56 & 74.31 & 70.09 & 71.65 & 74.97 & 72.76 \\
 & RAG  & 78.56 & 74.23 & 76.12 & 79.99 & 77.23 & \cellcolor{secondblue}\textbf{78.34} & \cellcolor{secondblue}\textbf{74.12} & \cellcolor{secondblue}\textbf{76.56} & \cellcolor{secondblue}\textbf{79.45} & \cellcolor{secondblue}\textbf{77.12} & 77.89 & 74.12 & 75.89 & 78.33 & 76.56 & \cellcolor{secondblue}\textbf{78.12} & \cellcolor{secondblue}\textbf{73.89} & \cellcolor{secondblue}\textbf{75.45} & \cellcolor{secondblue}\textbf{78.34} & \cellcolor{secondblue}\textbf{76.45} & \cellcolor{secondblue}\textbf{78.23} & 74.09 & \cellcolor{secondblue}\textbf{76.01} & 79.03 & \cellcolor{secondblue}\textbf{76.84} \\
 & Iteration  & \cellcolor{secondblue}\textbf{79.88} & \cellcolor{secondblue}\textbf{75.67} & \cellcolor{secondblue}\textbf{77.45} & \cellcolor{secondblue}\textbf{81.23} & \cellcolor{secondblue}\textbf{78.56} & 77.45 & 72.89 & 74.11 & 79.12 & 75.89 & \cellcolor{secondblue}\textbf{78.56} & \cellcolor{secondblue}\textbf{75.34} & \cellcolor{secondblue}\textbf{76.78} & \cellcolor{secondblue}\textbf{79.13} & \cellcolor{secondblue}\textbf{77.45} & 76.89 & 72.56 & 73.45 & 78.02 & 75.23 & 78.19 & \cellcolor{secondblue}\textbf{74.12} & 75.45 & \cellcolor{secondblue}\textbf{79.37} & 76.78 \\
 & RCBSF  & \cellcolor{bestblue}\textbf{81.01} & \cellcolor{bestblue}\textbf{82.99} & \cellcolor{bestblue}\textbf{87.43} & \cellcolor{bestblue}\textbf{87.25} & \cellcolor{bestblue}\textbf{84.67} & \cellcolor{bestblue}\textbf{80.41} & \cellcolor{bestblue}\textbf{82.46} & \cellcolor{bestblue}\textbf{86.21} & \cellcolor{bestblue}\textbf{86.48} & \cellcolor{bestblue}\textbf{83.89} & \cellcolor{bestblue}\textbf{80.01} & \cellcolor{bestblue}\textbf{82.69} & \cellcolor{bestblue}\textbf{83.99} & \cellcolor{bestblue}\textbf{87.11} & \cellcolor{bestblue}\textbf{83.45} & \cellcolor{bestblue}\textbf{80.15} & \cellcolor{bestblue}\textbf{82.51} & \cellcolor{bestblue}\textbf{80.31} & \cellcolor{bestblue}\textbf{87.71} & \cellcolor{bestblue}\textbf{82.67} & \cellcolor{bestblue}\textbf{80.40} & \cellcolor{bestblue}\textbf{82.66} & \cellcolor{bestblue}\textbf{84.49} & \cellcolor{bestblue}\textbf{87.14} & \cellcolor{bestblue}\textbf{83.67} \\
\midrule
\multirow{5}{*}{Mistral-7B}
 & Standard  & 70.45 & 65.23 & 67.89 & 70.23 & 68.45 & 69.12 & 64.33 & 66.45 & 68.58 & 67.12 & 67.23 & 62.45 & 64.12 & 67.56 & 65.34 & 68.56 & 64.12 & 66.23 & 68.21 & 66.78 & 68.84 & 64.03 & 66.17 & 68.64 & 66.92 \\
 & CoT  & 73.23 & 68.12 & 70.56 & 73.01 & 71.23 & 71.45 & 66.78 & 68.89 & 70.68 & 69.45 & 69.89 & 65.12 & 66.78 & 69.77 & 67.89 & 70.12 & 65.45 & 67.89 & 69.02 & 68.12 & 71.17 & 66.37 & 68.53 & 70.62 & 69.17 \\
 & RAG  & 77.56 & 72.89 & 74.34 & 77.89 & 75.67 & 75.23 & 70.45 & 72.56 & 74.68 & 73.23 & 73.45 & 68.89 & 70.12 & 73.78 & 71.56 & 74.23 & 69.56 & 71.45 & 74.56 & 72.45 & 75.12 & 70.45 & 72.12 & 75.23 & 73.23 \\
 & Iteration  & \cellcolor{secondblue}\textbf{79.67} & \cellcolor{secondblue}\textbf{75.12} & \cellcolor{secondblue}\textbf{76.89} & \cellcolor{secondblue}\textbf{79.84} & \cellcolor{secondblue}\textbf{77.88} & \cellcolor{secondblue}\textbf{77.45} & \cellcolor{secondblue}\textbf{72.89} & \cellcolor{secondblue}\textbf{74.56} & \cellcolor{secondblue}\textbf{77.78} & \cellcolor{secondblue}\textbf{75.67} & \cellcolor{secondblue}\textbf{75.23} & \cellcolor{secondblue}\textbf{70.89} & \cellcolor{secondblue}\textbf{72.45} & \cellcolor{secondblue}\textbf{75.23} & \cellcolor{secondblue}\textbf{73.45} & \cellcolor{secondblue}\textbf{77.12} & \cellcolor{secondblue}\textbf{72.34} & \cellcolor{secondblue}\textbf{74.23} & \cellcolor{secondblue}\textbf{77.23} & \cellcolor{secondblue}\textbf{75.23} & \cellcolor{secondblue}\textbf{77.37} & \cellcolor{secondblue}\textbf{72.81} & \cellcolor{secondblue}\textbf{74.53} & \cellcolor{secondblue}\textbf{77.52} & \cellcolor{secondblue}\textbf{75.56} \\
 & RCBSF  & \cellcolor{bestblue}\textbf{83.89} & \cellcolor{bestblue}\textbf{79.67} & \cellcolor{bestblue}\textbf{81.23} & \cellcolor{bestblue}\textbf{84.53} & \cellcolor{bestblue}\textbf{82.33} & \cellcolor{bestblue}\textbf{82.34} & \cellcolor{bestblue}\textbf{77.56} & \cellcolor{bestblue}\textbf{79.45} & \cellcolor{bestblue}\textbf{83.69} & \cellcolor{bestblue}\textbf{80.76} & \cellcolor{bestblue}\textbf{80.45} & \cellcolor{bestblue}\textbf{76.34} & \cellcolor{bestblue}\textbf{77.89} & \cellcolor{bestblue}\textbf{80.96} & \cellcolor{bestblue}\textbf{78.91} & \cellcolor{bestblue}\textbf{80.67} & \cellcolor{bestblue}\textbf{76.45} & \cellcolor{bestblue}\textbf{78.12} & \cellcolor{bestblue}\textbf{80.44} & \cellcolor{bestblue}\textbf{78.92} & \cellcolor{bestblue}\textbf{81.84} & \cellcolor{bestblue}\textbf{77.51} & \cellcolor{bestblue}\textbf{79.17} & \cellcolor{bestblue}\textbf{82.41} & \cellcolor{bestblue}\textbf{80.23} \\
\midrule
\multirow{5}{*}{LawLLM-7B}
 & Standard  & 72.12 & 67.34 & 69.12 & 72.34 & 70.23 & 70.89 & 65.56 & 67.45 & 71.66 & 68.89 & 69.34 & 64.89 & 66.23 & 69.34 & 67.45 & 70.12 & 65.45 & 67.12 & 69.79 & 68.12 & 70.62 & 65.81 & 67.48 & 70.78 & 68.67 \\
 & CoT  & 74.34 & 69.89 & 71.56 & 74.45 & 72.56 & 73.12 & 68.12 & 70.23 & 73.45 & 71.23 & 71.67 & 67.23 & 68.89 & 71.73 & 69.88 & 72.45 & 67.78 & 69.23 & 72.34 & 70.45 & 72.90 & 68.25 & 69.98 & 72.99 & 71.03 \\
 & RAG  & 79.12 & 74.89 & 76.56 & 79.23 & 77.45 & 78.56 & 73.89 & 75.89 & 79.22 & 76.89 & \cellcolor{secondblue}\textbf{77.23} & \cellcolor{secondblue}\textbf{72.56} & \cellcolor{secondblue}\textbf{74.34} & \cellcolor{secondblue}\textbf{77.23} & \cellcolor{secondblue}\textbf{75.34} & \cellcolor{secondblue}\textbf{79.34} & \cellcolor{secondblue}\textbf{74.89} & \cellcolor{secondblue}\textbf{76.45} & \cellcolor{secondblue}\textbf{79.56} & \cellcolor{secondblue}\textbf{77.56} & 78.56 & 74.06 & 75.81 & 78.81 & 76.81 \\
 & Iteration  & \cellcolor{secondblue}\textbf{80.89} & \cellcolor{secondblue}\textbf{76.45} & \cellcolor{secondblue}\textbf{78.23} & \cellcolor{secondblue}\textbf{80.91} & \cellcolor{secondblue}\textbf{79.12} & \cellcolor{secondblue}\textbf{80.12} & \cellcolor{secondblue}\textbf{75.56} & \cellcolor{secondblue}\textbf{77.34} & \cellcolor{secondblue}\textbf{81.22} & \cellcolor{secondblue}\textbf{78.56} & 76.12 & 71.45 & 72.78 & 76.57 & 74.23 & 77.89 & 73.45 & 75.12 & 78.02 & 76.12 & \cellcolor{secondblue}\textbf{78.76} & \cellcolor{secondblue}\textbf{74.23} & \cellcolor{secondblue}\textbf{75.87} & \cellcolor{secondblue}\textbf{79.18} & \cellcolor{secondblue}\textbf{77.01} \\
 & RCBSF  & \cellcolor{bestblue}\textbf{85.89} & \cellcolor{bestblue}\textbf{81.88} & \cellcolor{bestblue}\textbf{83.56} & \cellcolor{bestblue}\textbf{86.87} & \cellcolor{bestblue}\textbf{84.55} & \cellcolor{bestblue}\textbf{84.23} & \cellcolor{bestblue}\textbf{79.67} & \cellcolor{bestblue}\textbf{81.34} & \cellcolor{bestblue}\textbf{85.40} & \cellcolor{bestblue}\textbf{82.66} & \cellcolor{bestblue}\textbf{82.34} & \cellcolor{bestblue}\textbf{77.89} & \cellcolor{bestblue}\textbf{79.56} & \cellcolor{bestblue}\textbf{83.29} & \cellcolor{bestblue}\textbf{80.77} & \cellcolor{bestblue}\textbf{83.89} & \cellcolor{bestblue}\textbf{79.78} & \cellcolor{bestblue}\textbf{81.45} & \cellcolor{bestblue}\textbf{85.12} & \cellcolor{bestblue}\textbf{82.56} & \cellcolor{bestblue}\textbf{84.09} & \cellcolor{bestblue}\textbf{79.81} & \cellcolor{bestblue}\textbf{81.48} & \cellcolor{bestblue}\textbf{85.17} & \cellcolor{bestblue}\textbf{82.64} \\
\midrule
\multirow{5}{*}{LexiLaw-6B}
 & Standard  & 62.45 & 57.67 & 59.23 & 62.45 & 60.45 & 62.12 & 57.45 & 59.12 & 62.23 & 60.23 & 62.11 & 57.45 & 58.89 & 62.15 & 60.15 & 62.45 & 57.89 & 59.56 & 63.22 & 60.78 & 62.28 & 57.62 & 59.20 & 62.51 & 60.40 \\
 & CoT  & 64.78 & 59.89 & 61.56 & 65.33 & 62.89 & 64.23 & 59.23 & 61.12 & 63.86 & 62.11 & 63.45 & 58.78 & 60.12 & 63.45 & 61.45 & 64.23 & 59.89 & 61.45 & 64.67 & 62.56 & 64.17 & 59.45 & 61.06 & 64.33 & 62.25 \\
 & RAG  & \cellcolor{secondblue}\textbf{70.45} & \cellcolor{secondblue}\textbf{65.45} & \cellcolor{secondblue}\textbf{67.23} & \cellcolor{secondblue}\textbf{71.11} & \cellcolor{secondblue}\textbf{68.56} & \cellcolor{secondblue}\textbf{68.89} & \cellcolor{secondblue}\textbf{63.78} & \cellcolor{secondblue}\textbf{65.56} & \cellcolor{secondblue}\textbf{68.89} & \cellcolor{secondblue}\textbf{66.78} & 65.89 & 61.23 & 62.78 & 65.66 & 63.89 & \cellcolor{secondblue}\textbf{69.45} & \cellcolor{secondblue}\textbf{64.56} & \cellcolor{secondblue}\textbf{66.23} & \cellcolor{secondblue}\textbf{69.56} & \cellcolor{secondblue}\textbf{67.45} & \cellcolor{secondblue}\textbf{68.67} & \cellcolor{secondblue}\textbf{63.76} & \cellcolor{secondblue}\textbf{65.45} & \cellcolor{secondblue}\textbf{68.80} & \cellcolor{secondblue}\textbf{66.67} \\
 & Iteration  & 68.23 & 63.89 & 65.12 & 68.12 & 66.34 & 66.56 & 62.12 & 63.89 & 65.67 & 64.56 & \cellcolor{secondblue}\textbf{67.34} & \cellcolor{secondblue}\textbf{62.78} & \cellcolor{secondblue}\textbf{63.89} & \cellcolor{secondblue}\textbf{66.91} & \cellcolor{secondblue}\textbf{65.23} & 67.12 & 62.56 & 64.12 & 66.68 & 65.12 & 67.31 & 62.84 & 64.26 & 66.84 & 65.31 \\
 & RCBSF  & \cellcolor{bestblue}\textbf{76.89} & \cellcolor{bestblue}\textbf{71.67} & \cellcolor{bestblue}\textbf{73.45} & \cellcolor{bestblue}\textbf{77.51} & \cellcolor{bestblue}\textbf{74.88} & \cellcolor{bestblue}\textbf{75.23} & \cellcolor{bestblue}\textbf{69.89} & \cellcolor{bestblue}\textbf{71.67} & \cellcolor{bestblue}\textbf{76.13} & \cellcolor{bestblue}\textbf{73.23} & \cellcolor{bestblue}\textbf{71.89} & \cellcolor{bestblue}\textbf{67.23} & \cellcolor{bestblue}\textbf{68.89} & \cellcolor{bestblue}\textbf{71.51} & \cellcolor{bestblue}\textbf{69.88} & \cellcolor{bestblue}\textbf{73.23} & \cellcolor{bestblue}\textbf{68.45} & \cellcolor{bestblue}\textbf{70.12} & \cellcolor{bestblue}\textbf{73.08} & \cellcolor{bestblue}\textbf{71.22} & \cellcolor{bestblue}\textbf{74.31} & \cellcolor{bestblue}\textbf{69.31} & \cellcolor{bestblue}\textbf{71.03} & \cellcolor{bestblue}\textbf{74.56} & \cellcolor{bestblue}\textbf{72.30} \\
\midrule
\multirow{5}{*}{Qwen2.5-7B-Chat}
 & Standard  & 76.56 & 71.45 & 73.23 & 76.12 & 74.34 & 75.12 & 69.89 & 71.56 & 75.91 & 73.12 & 75.45 & 70.89 & 72.45 & 75.01 & 73.45 & 74.89 & 69.56 & 71.89 & 75.22 & 72.89 & 75.51 & 70.45 & 72.28 & 75.56 & 73.45 \\
 & CoT  & 79.12 & 74.34 & 75.89 & 78.21 & 76.89 & 77.45 & 72.34 & 74.12 & 77.89 & 75.45 & 77.12 & 72.56 & 73.89 & 76.91 & 75.12 & 77.34 & 72.12 & 74.34 & 77.12 & 75.23 & 77.76 & 72.84 & 74.56 & 77.53 & 75.67 \\
 & RAG  & 82.34 & 77.56 & 79.34 & 81.24 & 80.12 & \cellcolor{secondblue}\textbf{82.34} & \cellcolor{secondblue}\textbf{77.12} & \cellcolor{secondblue}\textbf{79.45} & \cellcolor{secondblue}\textbf{82.01} & \cellcolor{secondblue}\textbf{80.23} & 81.56 & 76.89 & 78.23 & 81.56 & 79.56 & \cellcolor{secondblue}\textbf{81.23} & \cellcolor{secondblue}\textbf{76.89} & \cellcolor{secondblue}\textbf{78.45} & \cellcolor{secondblue}\textbf{80.79} & \cellcolor{secondblue}\textbf{79.34} & \cellcolor{secondblue}\textbf{81.87} & 77.12 & \cellcolor{secondblue}\textbf{78.87} & \cellcolor{secondblue}\textbf{81.40} & \cellcolor{secondblue}\textbf{79.81} \\
 & Iteration  & \cellcolor{secondblue}\textbf{82.89} & \cellcolor{secondblue}\textbf{79.99} & \cellcolor{secondblue}\textbf{80.56} & \cellcolor{secondblue}\textbf{82.80} & \cellcolor{secondblue}\textbf{81.56} & 81.89 & 75.23 & 77.12 & 81.32 & 78.89 & \cellcolor{secondblue}\textbf{82.45} & \cellcolor{secondblue}\textbf{77.89} & \cellcolor{secondblue}\textbf{79.23} & \cellcolor{secondblue}\textbf{82.23} & \cellcolor{secondblue}\textbf{80.45} & 80.12 & 76.01 & 77.89 & 78.46 & 78.12 & 81.84 & \cellcolor{secondblue}\textbf{77.28} & 78.70 & 81.20 & 79.76 \\
 & RCBSF  & \cellcolor{bestblue}\textbf{83.01} & \cellcolor{bestblue}\textbf{85.99} & \cellcolor{bestblue}\textbf{90.43} & \cellcolor{bestblue}\textbf{91.57} & \cellcolor{bestblue}\textbf{87.75} & \cellcolor{bestblue}\textbf{83.41} & \cellcolor{bestblue}\textbf{85.46} & \cellcolor{bestblue}\textbf{89.21} & \cellcolor{bestblue}\textbf{90.05} & \cellcolor{bestblue}\textbf{87.03} & \cellcolor{bestblue}\textbf{82.98} & \cellcolor{bestblue}\textbf{85.68} & \cellcolor{bestblue}\textbf{87.01} & \cellcolor{bestblue}\textbf{91.89} & \cellcolor{bestblue}\textbf{86.89} & \cellcolor{bestblue}\textbf{83.15} & \cellcolor{bestblue}\textbf{85.52} & \cellcolor{bestblue}\textbf{83.31} & \cellcolor{bestblue}\textbf{91.33} & \cellcolor{bestblue}\textbf{85.83} & \cellcolor{bestblue}\textbf{83.14} & \cellcolor{bestblue}\textbf{85.66} & \cellcolor{bestblue}\textbf{87.49} & \cellcolor{bestblue}\textbf{91.21} & \cellcolor{bestblue}\textbf{86.87} \\
\bottomrule
\end{tabular}
}
\caption{Fine-grained Quality Metrics Breakdown. We report Clarity (\textbf{Cla}), Rigor (\textbf{Rig}), Balance (\textbf{Bal}), and Professionalism (\textbf{Pro}) for each dataset and their average (\textbf{ALL}). The \colorbox{bestblue}{\textbf{dark blue}} and \colorbox{secondblue}{\textbf{light blue}} cells indicate the best and second-best performance within each model group.}
\label{tab:full_quality_metrics}
\end{table*}

\begin{figure*}[htb]
    \centering
    \includegraphics[width=0.9\linewidth]{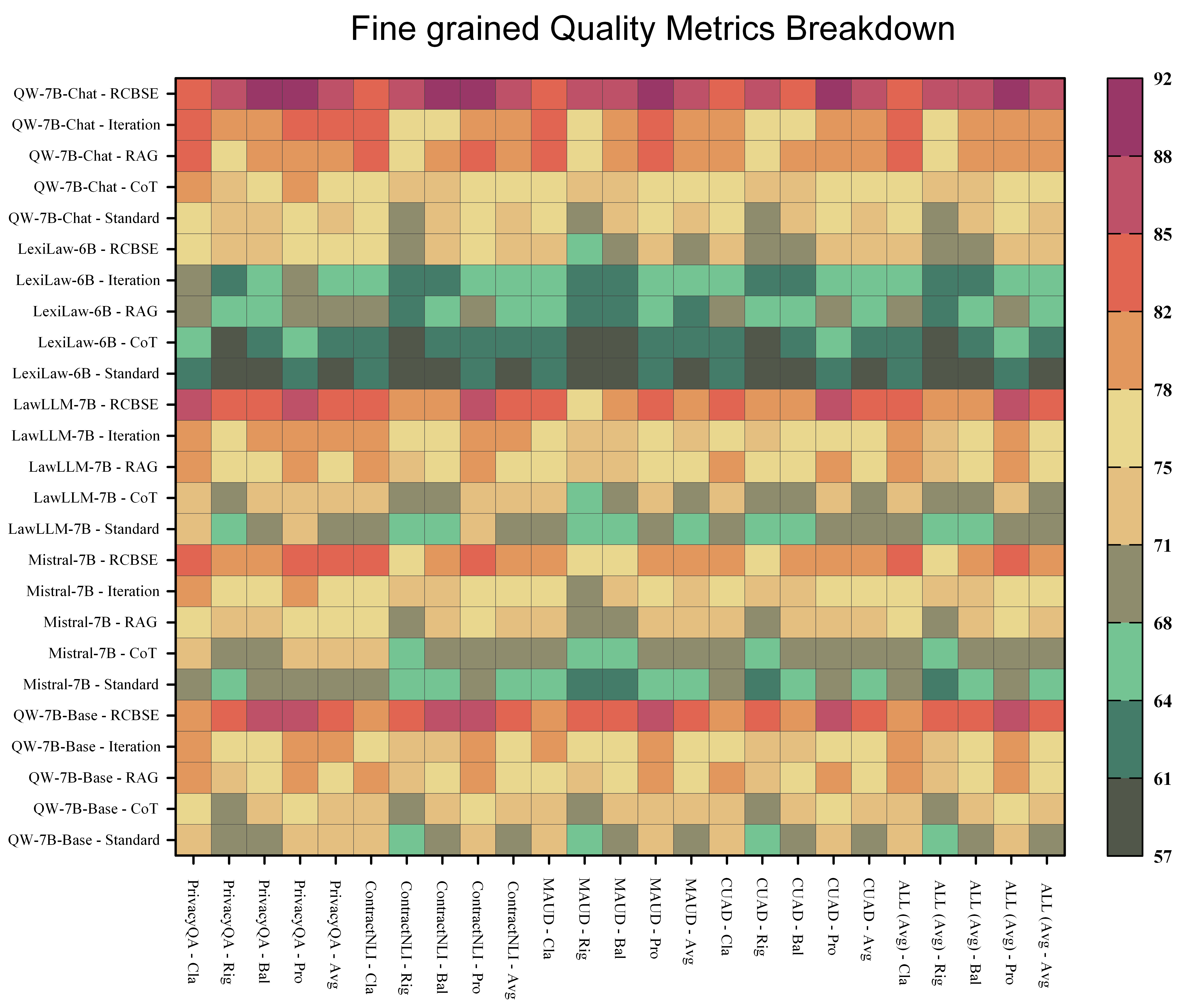}
    \caption{Heatmap of Fine-Grained Quality Metrics. Multi-model-group variants’ performance across multi-tasks. Dark magenta is the best in group, dark green is the worst in group. Metric values range from 57--92.}
    \label{fig:Heatmap of Fine-Grained}
\end{figure*}

In this section, we provide a granular analysis of the experimental results, focusing on the fine-grained quality metrics, and the component-wise ablation contributions.

\subsection{Fine-grained Quality Analysis}
\label{subsec:quality_breakdown}

While the main paper focuses on the Risk Resolution Rate (RRR), the linguistic and legal quality of the generated text is equally critical. We break down the Contract Quality (CQ) into four orthogonal dimensions: \textit{Clarity} ($\mathcal{C}$), \textit{Rigor} ($\mathcal{R}$), \textit{Balance} ($\mathcal{B}$), and \textit{Professionalism} ($\mathcal{P}$).

To formally evaluate the improvement, we define the \textit{Quality Gain Tensor} $\Delta \mathbf{Q}$ for a model $M$ under the \textsc{RCBSF} framework compared to the Baseline ($B$) as:

\begin{equation}
\Delta \mathbf{Q}_{M} = \frac{1}{|\mathcal{D}|} \sum_{x \in \mathcal{D}} \left( \Phi(x_{RCBSF}) - \Phi(x_{Base}) \right)
\end{equation}

where $\Phi: \mathcal{X} \to \mathbb{R}^4$ maps the contract text to the score vector $[\mathcal{C}, \mathcal{R}, \mathcal{B}, \mathcal{P}]^\top$.

\paragraph{Analysis of the Contract Quality results:}
Table~\ref{tab:full_quality_metrics} and Figure~\ref{fig:Heatmap of Fine-Grained} present the comprehensive quantitative results and their corresponding visual distributions. This dual analysis reveals three key insights:
\begin{enumerate}
    \item \textbf{Rigor is the primary beneficiary:} Across all backbones, the \textit{Rigor} metric sees the most significant improvement. For instance, with LawLLM-7B, the Rigor score improves from $65.81$ (Standard) to $79.81$ (\textsc{RCBSF}). This logical enhancement is visually striking in Figure~\ref{fig:Heatmap of Fine-Grained}, where the Rigor columns transition from dark green (worst) in baseline methods to deep magenta (best) in the \textsc{RCBSF} rows, verifying that the LVA's constraints effectively curb the loose logic often observed in unguided LLMs.
    \item \textbf{Stability across Backbones:} Even for weaker backbones like LexiLaw-6B, \textsc{RCBSF} boosts the average quality from $60.40$ to $72.30$. The heatmap corroborates this stability, displaying a consistent high-intensity color band for \textsc{RCBSF} across all model groups compared to the patchy performance of baselines. This suggests that the Stackelberg game mechanism is model-agnostic and can force even smaller models to adhere to strict legal standards.
    \item \textbf{Peak Performance:} The strongest configuration, QW-7B-Chat + RCBSF, achieves state-of-the-art results with an average CQ of $\mathbf{86.87}$. As illustrated by the darkest magenta blocks in Figure~\ref{fig:Heatmap of Fine-Grained}, this configuration significantly outperforms the RAG baseline ($79.81$) and approaches expert human levels ($>90$).
\end{enumerate}

\subsection{Ablation Study Breakdown}
\label{subsec:ablation_breakdown}

To understand the contribution of the 5-dimensional constraint structure, we conducted an ablation study by selectively masking the inputs to the CRA. We define two error modes:
\begin{enumerate}
    \item \textbf{Hallucination Rate (HR):} The CRA modifies a clause that was NOT risky.
    \item \textbf{Fix Failure Rate (FFR):} The CRA attempts to fix a risk but fails to resolve it substantively.
\end{enumerate}

\begin{table}[h]
\centering
\small
\begin{tabular}{l|cc|c}
\toprule
\textbf{Configuration} & \textbf{HR ($\downarrow$)} & \textbf{FFR ($\downarrow$)} & \textbf{RRR ($\uparrow$)} \\
\midrule
\textbf{Full \textsc{RCBSF}} & \textbf{4.23\%} & \textbf{7.35\%} & \textbf{84.21\%} \\
\midrule
w/o Evidence ($e_k$) & 18.68\% & 11.24\% & 76.25\% \\
w/o Suggestion ($s_k$) & 5.19\% & 19.86\% & 75.14\% \\
w/o Location ($l_k$) & 12.47\% & 9.54\% & 78.19\% \\
\bottomrule
\end{tabular}
\caption{Component-wise Ablation. Removing specific dimensions from the Leader's instruction leads to distinct failure modes.}
\label{tab:ablation_component}
\end{table}

\paragraph{Analysis.} 
As shown in Table~\ref{tab:ablation_component}, the results of the component-wise ablation study yield several key observations regarding the contribution of each instruction dimension:

\begin{itemize}
    \item \textbf{Effect of Evidence ($e_k$):} The removal of explicit evidence ($e_k$) leads to a marked spike in the \textit{Hallucination Rate} ($4.23\% \to 18.68\%$). Without the anchor text to ground the generation, the CRA frequently modifies unrelated clauses that share similar keywords, thereby compromising the semantic integrity of the contract.
    \item \textbf{Effect of Suggestion ($s_k$):} Excluding actionable suggestions ($s_k$) results in a drastic increase in the \textit{Fix Failure Rate} ($7.35\% \to 19.86\%$). While the model successfully identifies the risk (indicated by the Issue field), it tends to apply non-committal or generic mitigations (e.g., merely adding "to the extent reasonable") rather than enforcing necessary hard constraints such as liability caps.
    \item \textbf{Effect of Location ($l_k$):} Omitting location information ($l_k$) significantly degrades performance, specifically increasing the \textit{Hallucination Rate} ($4.23\% \to 12.47\%$). This indicates that precise localization cues act as essential boundary constraints; without them, the model struggles to pinpoint the correct target scope, often leading to erroneous modifications in non-target sections or the fabrication of contextually irrelevant clauses.
\end{itemize}


\section{Qualitative Case Studies}
\label{sec:appendix_cases}

To provide an intuitive understanding of the \textsc{RCBSF} framework, we present qualitative examples tracing the full lifecycle of a contract revision. We illustrate how the \textit{Global Prescriptive Agent (Leader)} enforces constraints and how the \textit{Follower System} (CRA + Inner LVA) iteratively refines the text. Finally, we analyze typical failure modes to discuss current limitations.

\subsection{Success Cases (End-to-End Flow)}
\label{subsec:success_cases}

In Figure~\ref{fig:case_study_1} and Figure~\ref{fig:case_study_2}, we display two distinct negotiation scenarios. 

\paragraph{Case 1: The Unlimited Liability Trap.} 
This case (Figure~\ref{fig:case_study_1}) demonstrates the correction of a high-risk liability clause.
\begin{enumerate}
    \item \textbf{Input:} The original text contains a "blanket indemnity" clause exposing the provider to unlimited risk.
    \item \textbf{Leader Analysis:} The LVA identifies this as a \textbf{Category A} risk (Q2=A). It outputs a structured tuple specifying the \textit{Evidence} ("indemnify... from any and all claims") and a concrete \textit{Suggestion} ("Insert a liability cap").
    \item \textbf{Revision Process:} 
    \begin{itemize}
        \item \textit{Revision \#1:} The CRA softens the language but forgets the specific cap.
        \item \textit{Inner Audit:} The LVA flags that the "financial cap is still missing."
        \item \textit{Final Revision:} The CRA incorporates the "12 months' fees" cap, satisfying the Leader's constraint.
    \end{itemize}
\end{enumerate}

\begin{figure*}[htb]
    \centering
    \begin{tcolorbox}[colback=white, colframe=blue!40!black, title=\textbf{Case Study 1: Mitigating Unlimited Liability Risks}]
    \small
    \textbf{1. Original Contract Segment (Input):}
    \texttt{"The Provider shall indemnify and hold the Client harmless from and against \colorbox{red!20}{any and all claims, losses, damages}, liabilities, and expenses arising out of or in connection with the Services, regardless of cause."}
    
    \vspace{0.2cm}
    \textbf{2. Leader Agent Analysis (Strategic Instruction):}
    \begin{itemize}
        \item \textbf{Category:} \textsc{Excessive Liability}
        \item \textbf{Location:} Section 9.1 (Indemnification)
        \item \textbf{Evidence:} \textit{"...indemnify... from any and all claims... regardless of cause."}
        \item \textbf{Issue:} Creates unlimited exposure for the Provider; lacks a monetary cap or exclusion for indirect damages.
        \item \textbf{Suggestion:} \colorbox{yellow!30}{\textbf{Limit liability to the fees paid in the preceding 12 months} and exclude consequential damages.}
        \item \textbf{Q-Score:} \{Q1: B, \textbf{Q2: A (High)}, Q3: B, Q4: A\}
    \end{itemize}

    \vspace{0.2cm}
    \textbf{3. Follower System Execution (Iterative):}
    
    \textit{[Revision \#1 - Attempt]:}
    \texttt{"The Provider shall indemnify the Client for direct losses arising from the Provider's negligence..."} 
    \hfill \textit{\textcolor{red}{$\to$ Rejected by Inner Audit (Missing Cap).}}

    \textit{[Final Revision - Convergence]:}
    \texttt{"The Provider shall indemnify the Client for direct losses... provided that \colorbox{green!20}{the Provider's total aggregate liability shall not exceed the total fees paid} by the Client in the twelve (12) months preceding the claim. In no event shall Provider be liable for indirect or consequential damages."}
    \end{tcolorbox}
    \caption{A complete revision flow for a Liability Clause. The Leader's specific suggestion (highlighted in yellow) serves as a hard constraint that the Follower eventually satisfies in the final output (green).}
    \label{fig:case_study_1}
\end{figure*}

\paragraph{Case 2:  Multi-Objective Risk Resolution in Commercial Contracts.} 
\label{subsec:case_study_multi}

To illustrate the capability of RCBSF in handling complex commercial agreements, Figure~\ref{fig:case_study_2} visualizes the revision trajectory of a \textit{Software Development Services Agreement}. The input text contains three distinct high-stakes risks: (1) \textit{Unilateral Termination}, (2) \textit{Vendor IP Retention} (replacing standard liability issues with a critical asset ownership risk), and (3) \textit{Ambiguous Payment Terms}.

The revision process adheres to the hierarchical Stackelberg dynamics defined in our implementation:
\begin{enumerate}
    \item \textbf{Global Prescription (Outer Loop):} The \textbf{Global Prescriptive Agent (GPA)} decomposes the contract text, extracting a 5-dimensional constraint vector for each of the three identified risks.
    \item \textbf{Execution \& Audit (Inner Loop):} The Follower system (CRA \& LVA) engages in a 3-round iterative revision process. 
    \item \textbf{Trajectory Analysis:}
    \begin{itemize}
        \item \textbf{Iteration \#1:} The CRA successfully resolves the Payment and Termination risks but fails to address the IP ownership transfer, merely polishing the language.
        \item \textbf{Fusion Feedback:} The LVA triggers a negative reward signal regarding the IP clause. The Fusion module generates a specific hint: \textit{"Priority: Reassign Deliverable ownership to Client."}
        \item \textbf{Iteration \#3:} The CRA integrates all constraints, producing a finalized clause that balances fairness, clarity, and asset protection.
    \end{itemize}
\end{enumerate}

\begin{figure*}[htb]
    \centering
    \begin{tcolorbox}[colback=white, colframe=blue!40!black, title=\textbf{Case Study 2: Simultaneous Resolution of Three Commercial Risks}]
    \small
    
    \textbf{\large Phase 1: Input Contract (Raw Text)}
    \begin{quote}
    "The Company may terminate this Agreement \colorbox{red!15}{at any time, effective immediately}. Regarding intellectual property, \colorbox{red!15}{Vendor shall retain all right, title, and interest} in and to any software or deliverables created hereunder. Client shall \colorbox{red!15}{pay all invoices submitted by Vendor}."
    \end{quote}
    
    \vspace{0.15cm}
    \rule{\linewidth}{0.5pt}
    \vspace{0.15cm}

    \textbf{\large Phase 2: Global Prescriptive Agent (GPA) Analysis}
    \textit{The GPA generates a 5-dimensional risk budget for each detected issue:}

    \begin{itemize}
        \item \textbf{[Risk 1] Unilateral Termination (Fairness)}
        \begin{itemize}
            \item \textit{Evidence:} "...terminate... at any time, effective immediately."
            \item \textit{Issue:} Immediate termination disrupts project continuity and lacks procedural fairness.
            \item \textit{Suggestion:} \colorbox{yellow!25}{Require 30-day prior written notice for convenience termination.}
        \end{itemize}
        \item \textbf{[Risk 2] IP Retention (Asset Ownership)}
        \begin{itemize}
            \item \textit{Evidence:} "Vendor shall retain all right, title, and interest..."
            \item \textit{Issue:} Client pays for development but fails to own the work product (Work Made for Hire).
            \item \textit{Suggestion:} \colorbox{yellow!25}{Transfer ownership of Deliverables to Client upon payment.}
        \end{itemize}
        \item \textbf{[Risk 3] Ambiguous Payment (Financial)}
        \begin{itemize}
            \item \textit{Evidence:} "...pay all invoices submitted..."
            \item \textit{Issue:} Lacks specific payment terms (e.g., Net 30), creating cash flow uncertainty.
            \item \textit{Suggestion:} \colorbox{yellow!25}{Specify payment due date (e.g., within 30 days of receipt).}
        \end{itemize}
    \end{itemize}

    \vspace{0.15cm}
    \rule{\linewidth}{0.5pt}
    \vspace{0.15cm}

    \textbf{\large Phase 3: Follower System Execution (Inner Loop Dynamics)}
    
    \textbf{$\triangleright$ Iteration \#1 (Initial Revision by CRA):}
    \textit{The agent fixes Termination and Payment but misses the IP nuance.}
    \begin{quote}
    "The Company may terminate this Agreement upon \colorbox{green!15}{30 days' prior written notice}. Vendor \colorbox{orange!15}{retains ownership of the software} but grants Client a license to use it. Client shall pay invoices \colorbox{green!15}{within thirty (30) days of receipt}."
    \end{quote}
    
    \textbf{$\rightarrow$ LVA Audit / Fusion Hint:} 
    \textit{"Risk 1 (Termination) \& Risk 3 (Payment) Resolved ($Q=A$). \textbf{Critical Failure on Risk 2 (IP):} 'License to use' is insufficient for custom software. \textbf{Instruction:} Must assign full title to Client."}
    
    \vspace{0.15cm}
    
    \textbf{$\triangleright$ Iteration \#3 (Final Converged Output):}
    \textit{The CRA applies the 'Work Made for Hire' logic mandated by the Fusion Hint.}
    \begin{quote}
    "The Company may terminate this Agreement without cause upon providing \colorbox{green!15}{at least thirty (30) days' prior written notice}. Vendor agrees that all Deliverables shall be considered \colorbox{green!15}{'Work Made for Hire' and hereby assigns all right, title, and interest} to Client. Client shall pay all undisputed invoices \colorbox{green!15}{within thirty (30) days} of receipt."
    \end{quote}
    
    \end{tcolorbox}
    \caption{Detailed visualization of the multi-risk resolution in a commercial software context. The \textbf{GPA} successfully identifies three disparate risks (Fairness, IP, Financial). The \textbf{Inner Loop} demonstrates the system's self-correction capability: while Iteration \#1 fails to fully transfer IP rights, the LVA's audit feedback forces the CRA to execute a legally binding assignment in Iteration \#3, satisfying all Leader constraints simultaneously.}
    \label{fig:case_study_2}
\end{figure*}

\subsection{Failure Case Analysis}
\label{subsec:failure_cases}

Despite the strong performance, \textsc{RCBSF} exhibits specific failure modes, primarily driven by context limitations and jurisdictional ambiguity.

\paragraph{Failure Mode A: The "Lost-in-the-Middle" Hallucination.}
In extremely long documents ($>15,000$ tokens), the Leader Agent occasionally "hallucinates" risks by misattributing Evidence from one section to another.

\begin{itemize}
    \item \textbf{Scenario:} A Merger Agreement where "Definition of Material Adverse Effect" (Page 5) excludes pandemics, but "Closing Conditions" (Page 80) references it.
    \item \textbf{Model Error:} The Leader claims \textit{"Missing Pandemic Exclusion"} because it fails to retrieve the definition from Page 5 due to attention decay over long contexts.
    \item \textbf{Result:} The CRA adds a redundant exclusion clause, lowering the \textit{Conciseness} score, although the legal risk is technically resolved.
\end{itemize}

\paragraph{Failure Mode B: Jurisdictional Overfitting.}
The model is predominantly trained on US-centric English corpora (CUAD, MAUD), leading to bias when handling Civil Law contracts.

\begin{itemize}
    \item \textbf{Input:} A contract governed by German Law (\textit{BGB}).
    \item \textbf{Leader Suggestion:} \textit{"Add a clause explicitly stating that 'consideration' has been exchanged to ensure validity."}
    \item \textbf{Analysis:} Under German law, "consideration" is not a requirement for contract validity (unlike in US/UK Common Law).
    \item \textbf{Outcome:} The system inserts a legally unnecessary "Consideration Clause," revealing a lack of jurisdiction-specific legal reasoning. This highlights the need for a "Jurisdiction Awareness" parameter in future iterations of the $h$ vector.
\end{itemize}

\section{Prompt Engineering}
\label{sec:appendix_prompt}

To ensure reproducibility and transparency, we provide the exact system prompts used for the Global Prescriptive Agent (Leader), the Q-Score Assessment mechanism, and the Constrained Revision Agent (Follower). The design of these prompts is central to the \textsc{RCBSF} framework's ability to balance risk mitigation with semantic preservation.


\subsection{Global Prescriptive Agent (Leader)}
\label{subsec:prompt_leader}

The Leader Agent is responsible for the initial risk auditing. Unlike standard find and fix prompts, we enforce a strict 5-Dimensional Output Structure (see Figure~\ref{fig:prompt_leader}) to ground the model's reasoning:
\begin{itemize}
    \item \textbf{Category:} Provides a high-level taxonomy for clustering risks.
    \item \textbf{Location:} Forces the model to perform "retrieval" within the context, mitigating hallucination by requiring pointer references.
    \item \textbf{Evidence:} Serves as a fact-checking mechanism. The model must quote the exact text causing the risk.
    \item \textbf{Issue:} Encodes the "Chain-of-Thought" reasoning, explaining \textit{why} the evidence constitutes a risk.
    \item \textbf{Suggestion:} Provides the actionable instruction for the Follower, decoupling the "what to do" from the "how to do it."
\end{itemize}

\begin{figure*}[htb] 
    \centering
    \begin{tcolorbox}[colback=white, colframe=gray!50!black, arc=0mm, title=\textbf{Prompt for Global Prescriptive Agent (Risk Extraction)}]
\small
\textbf{System Instruction:}
You are a strictly logical contract risk auditor. Please review the contract text below and extract risk points.
You must analyze from 5 dimensions for each risk:
\begin{enumerate}
    \item \textbf{category}: Specific risk classification (At least ten words or more, describe the risk category as detailed as possible).
    \item \textbf{location}: Where this risk appears (e.g., "Section 1.2").
    \item \textbf{evidence}: Original text quote supporting this risk. If missing, state "Missing clause".
    \item \textbf{issue}: Specific description of what is wrong (the defect, ambiguity, or unfairness).
    \item \textbf{suggestion}: Actionable advice on how to modify/add clause text.
\end{enumerate}

\textbf{Output Requirements:}
\begin{itemize}
    \item Output must be a JSON object containing a list "risk\_categories".
    \item If the text is very short, try to identify at least 8-15 potential risks/missing elements.
    \item Even for short texts or definitions, analyze strictly (e.g., Are definitions vague? Is the scope clear?).
\end{itemize}

\textbf{Input Context:}
Contract text:
\texttt{<<<CONTRACT>>>}
\{CONTRACT\_TEXT\}
\texttt{<<<END>>>}

Please only output JSON.
    \end{tcolorbox}
    \caption{The structured prompt used by the Leader Agent to extract the 5-dimensional risk tuple. This structure forces the model to ground its suggestions in specific evidence.}
    \label{fig:prompt_leader}
\end{figure*}

\begin{figure*}[htb]
    \centering
    \begin{tcolorbox}[
        enhanced,
        colback=white,
        colframe=teal!50!black, 
        arc=2mm,
        title=\textbf{\large Prompt for Local Verification Agent (Adversarial Auditor)}
    ]
    \small
    \textbf{\textcolor{teal!50!black}{[System Role Definition]}}
    You are the \textbf{Inner Auditor} (LVA). Your function is to calculate the \textit{Residual Risk} $\mathcal{R}(c_i, a_i)$ of the generated text. You must act logically and critically. Do not be lenient.

    \vspace{0.2cm}
    \textbf{\textcolor{teal!50!black}{[Audit Logic]}}
    Compare the \texttt{Revised\_Contract} against the \texttt{Risk\_Definition\_List}. 
    For every clause, determine if the risk is: \texttt{RESOLVED}, \texttt{PARTIALLY\_RESOLVED}, or \texttt{UNRESOLVED}.

    \rule{\linewidth}{0.5pt}

    \textbf{\textcolor{teal!50!black}{[Structured Output Requirement]}}
    Return a strictly formatted JSON object. This will be used to calculate the Loss Function.
    
    \begin{verbatim}
{
  "audit_report": [
    {
      "risk_id": "R_01 (Liability)",
      "status": "UNRESOLVED",
      "severity_weight": 0.9,
      "location_quote": "...liability shall not exceed...",
      "issue_description": "Cap is still ambiguous.",
      "gradient_feedback": "Specify exact dollar amount or fee multiplier."
    },
    ...
  ],
  "global_safety_score": 0.45
}
    \end{verbatim}
    \end{tcolorbox}
    \caption{The LVA Prompt. By enforcing a structured JSON output containing severity weights and gradient feedback, the Auditor transforms qualitative legal analysis into quantitative signals for the next optimization round.}
    \label{fig:prompt_lva}
\end{figure*}

\subsection{Q-Score Assessment Mechanism}
\label{subsec:prompt_qscore}

To quantify the severity of each identified risk, we employ a dedicated quantization prompt. As shown in Figure~\ref{fig:prompt_qscore}, the model evaluates each risk across four specific orthogonal dimensions ($Q1 \dots Q4$). This quantization allows the framework to construct the weighted risk instruction vector $h$.

\begin{itemize}
    \item \textbf{Q1 (Validity/Compliance):} Assesses the risk of contract invalidation or regulatory penalties.
    \item \textbf{Q2 (Liability Scope):} Measures the financial exposure (e.g., unlimited liability vs. capped).
    \item \textbf{Q3 (Control Allocation):} Evaluates the balance of power (e.g., unilateral termination rights).
    \item \textbf{Q4 (Remediability):} Estimates the cost and difficulty of fixing the issue if the risk materializes.
\end{itemize}

\subsection{Constrained Revision Agent (CRA)}
\label{subsec:prompt_CRA}

The Constrained Revision Agent (CRA) functions as the strategic follower and execution engine within the lower level of the Risk-Constrained Bilevel Stackelberg Framework (RCBSF). Unlike standard generative models that rely solely on semantic probability, the CRA operates as a rational optimizer in a non-cooperative game setting. Its primary objective is to generate a contract variant $x'$ that minimizes the composite loss function defined in Eq.~\ref{eq:loss_func}:
\begin{equation}
    \min_{x'} \mathbb{E} \left[ \mathcal{R}(x', a_{i}) + \lambda_1 \mathcal{C}(x', x) - \lambda_2 \mathcal{S}(x') \right]
    \label{eq:loss_func}
\end{equation}
where $\mathcal{R}$ denotes the residual risk detected by the auditor, $\mathcal{C}$ represents the consumption cost (edit distance from the original contract), and $\mathcal{S}$ is the semantic quality reward.

As illustrated in Figure~\ref{fig:prompt_CRA}, the CRA receives a structured composite instruction designed to guide this optimization:
\begin{itemize}
    \item \textbf{Global Constraints (\texttt{OUTER\_HINT}):} Derived from the Leader Agent (GPA), this provides the high-level risk budget and mandatory clauses (e.g., governing law, liability caps).
    \item \textbf{Local Gradients (\texttt{INNER\_HINT}):} Feedback from the Local Verification Agent (LVA) acting as an adversarial critic, pointing out specific residual risks in the previous iteration (Generation $N$).
    \item \textbf{Budget Control (\texttt{FUSION\_HINT}):} Instructions to manage the trade-off between risk mitigation and text preservation to prevent excessive modification.
\end{itemize}

\paragraph{Handling Local Optima.} A critical challenge in iterative LLM generation is the lazy revision phenomenon, where the model converges to a local optimum by returning the original text to avoid introducing new errors. To counteract this, we implement a Dynamic Force Rewrite Mechanism. As shown in Figure~\ref{fig:prompt_force}, when the edit distance between iterations is zero ($EditDistance(x_{t}, x_{t-1}) == 0$), the system injects a high-priority adversarial prompt to force a substantive modification, pushing the agent out of the saddle point.

\subsection{Local Verification Agent (LVA)}
\label{subsec:prompt_lva}

The Local Verification Agent (LVA) serves as the inner-loop auditor within the Stackelberg game. It is responsible for computing the residual risk $\mathcal{R}(\mathfrak{c}_{j},a_{i})$ at each iteration step. Unlike the global leader which sets the direction, the LVA performs a fine-grained, clause-level inspection of the \textit{revised} contract (Gen-$n$ Contract).

This agent operates in a Multi-Round Revision-Audit Interaction mechanism. In each epoch, the LVA analyzes the output of the CRA and generates a structured Gen-$n$ Audit Report. This report maps detected issues to specific risk categories and weights, providing the necessary gradient signal for the CRA to adjust its next generation. This adversarial dynamic ensures that hidden risks are exposed and resolved dynamically, addressing the Incomplete Information and Hidden Risks issues prevalent in baseline methods.

The LVA's prompt (see Figure~\ref{fig:prompt_lva}) strictly enforces a structured JSON output to facilitate automated parsing and weight mapping  in the next optimization cycle.

\begin{figure*}[htb]
    \centering
    \begin{tcolorbox}[
        enhanced,
        colback=white, 
        colframe=blue!90!black, 
        arc=2mm, 
        boxrule=1.5pt,
        drop shadow,
        title=\textbf{\large Prompt Construction for Constrained Revision Agent (CRA)}
    ]
    \small
    \textbf{\textcolor{blue!90!black}{[System Role Definition]}}
    You are the \textbf{Strategic Follower} (Optimizer) in a Stackelberg game for legal contract revision. Your goal is not just to write text, but to solve a bilevel optimization problem:
    \[ \min \text{Total Loss} = \text{Residual Risk}(\mathcal{R}) + \text{Modification Cost}(\mathcal{C}) - \text{Quality Reward}(\mathcal{S}) \]
    You must strictly adhere to the constraints passed by the Leader (GPA) and the Auditor (LVA).

    \rule{\linewidth}{0.5pt} 

    \textbf{\textcolor{blue!90!black}{[Input 1: Global Strategic Constraints (from Leader)]}}
    \textit{Source: Outer\_Hint}
    \begin{itemize}
        \item \textbf{Risk Budget}: Strict. No high-risk clauses allowed.
        \item \textbf{Mandatory Requirements}: Governing Law must be "Delaware"; Liability Cap must be limited to "12 months fees".
        \item \textbf{Key Entity}: "Alpha Product" licensing.
    \end{itemize}

    \textbf{\textcolor{blue!90!black}{[Input 2: Optimization Constraints (Budget)]}}
    \textit{Source: Fusion\_Hint}
    \begin{itemize}
        \item \textbf{Edit Distance Limit}: $\le 15\%$ deviation from original structure.
        \item \textbf{Style}: Formal legal English, maintaining original definitions.
    \end{itemize}

    \textbf{\textcolor{blue!90!black}{[Input 3: Gradient Feedback (from Auditor)]}}
    \textit{Source: Inner\_Hint (Previous Iteration N)}
    \texttt{> Critical Risk Detected in Section 4.2: "Indemnification clause is missing IP infringement coverage."}
    \texttt{> Suggestion: Insert standard IP indemnification language.}

    \rule{\linewidth}{0.5pt}

    \textbf{\textcolor{red!60!black}{[Execution Task]}}
    Output the full revised contract. Do not explain. Do not summarize.
    \end{tcolorbox}
    \caption{The composite prompt mechanism for the CRA. It translates the abstract Stackelberg theoretical constraints into executable instructions, combining global strategy (Leader), resource budget (Fusion), and specific risk gradients (Auditor).}
    \label{fig:prompt_CRA}
\end{figure*}


\begin{figure*}[htb]
    \centering
    \begin{tcolorbox}[
        enhanced,
        colback=red!3!white, 
        colframe=red!80!black, 
        arc=1.5mm, 
        boxrule=2.5pt, 
        title={\Large \textbf{$\triangle$ Dynamic Intervention: Anti-Lazy Mechanism}}, 
        fonttitle=\bfseries
    ]
    \small
    \textbf{Trigger Condition:} $\text{EditDistance}(x_{t}, x_{t-1}) < \epsilon$ \par
    \textit{(Indicator: Model has converged prematurely with no effective revisions)}

    \vspace{0.3cm}
    \textbf{\textcolor{red!90!black}{[Adversarial Injection Prompt]}}
    \begin{ttfamily}
    "SYSTEM ALERT: DETECTED SADDLE POINT IN OPTIMIZATION."
    \newline
    The previous output is identical to the input — you are NOT optimizing the objective function.
    \end{ttfamily}

    \vspace{0.2cm}
    \textbf{MANDATORY OVERRIDE (Non-Negotiable):}
    \begin{enumerate}
        \item You MUST rewrite the content in Section \texttt{[Detected\_Risk\_Location]}.
        \item Mere wording adjustments are insufficient: you MUST alter the logical framework to mitigate \texttt{[Risk\_Category]}.
        \item Penalty consequence for non-compliance: $\text{Penalty Score} = \infty$ (immediate optimization failure).
    \end{enumerate}
    \end{tcolorbox}
    \caption{The Force Rewrite injection mechanism. This acts as a perturbation vector to push the CRA out of local optima (i.e., lazy revisions) — triggered when the system detects zero effective edits despite unresolved risks.}
    \label{fig:prompt_force}
\end{figure*}

\begin{figure*}[htb] 
    \centering
    \begin{tcolorbox}[colback=white, colframe=gray!50!black, arc=0mm, title=\textbf{Prompt for Q-Score Assessment (Abbreviated)}]
\small
\textbf{System Instruction:}
Below is a contract clause (or a risk point). Based solely on this content, please choose an option from A/B/C for each of the 4 dimensions.

\textbf{Question 1: Nature of Legal Consequences (Validity / Compliance)}
\begin{itemize}
    \item \textbf{A (High):} Likely to invalidate the contract, trigger criminal liability, or result in license revocation.
    \item \textbf{B (Medium):} Clause partially invalid or subject to administrative fines/rectification.
    \item \textbf{C (Low):} Disputes over interpretation; does not affect validity or involve penalties.
\end{itemize}

\textbf{Question 2: Scope and Limit of Liability (Exposure)}
\begin{itemize}
    \item \textbf{A (High):} No liability cap; covers indirect losses; "bear all consequences."
    \item \textbf{B (Medium):} Cap exists but is high (>100\% price); includes some third-party claims.
    \item \textbf{C (Low):} Clear cap ($\le$100\% price); direct losses only; excludes consequential damages.
\end{itemize}

\textbf{Question 3: Allocation of Rights, Obligations, and Control}
\begin{itemize}
    \item \textbf{A (High):} Counterparty has unilateral authority (termination, pricing) with no veto/remedy.
    \item \textbf{B (Medium):} Unilateral rights exist but require notice, cure period, or objective conditions.
    \item \textbf{C (Low):} Rights are balanced; major matters require joint consent.
\end{itemize}

\textbf{Question 4: Remediability and Long-Term Impact}
\begin{itemize}
    \item \textbf{A (High):} Difficult to restore state; permanent loss of IP or trade secrets.
    \item \textbf{B (Medium):} Remediable via substantial renegotiation or multi-layer approvals.
    \item \textbf{C (Low):} Remediable via simple supplementary agreement; short-term impact.
\end{itemize}

\textbf{Output Format:}
Output \textbf{only one JSON object}: \texttt{\{"Q1": "A", "Q2": "B", "Q3": "C", "Q4": "C"\}}
    \end{tcolorbox}
    \caption{The definition of the Q-Score metrics. Each dimension is graded on a 3-point scale to compute the final risk weight.}
    \label{fig:prompt_qscore}
\end{figure*}

\end{document}